\documentclass{article}

  \usepackage{omniweaving_arxiv,times}

\usepackage[utf8]{inputenc} 
\usepackage[T1]{fontenc}    
\makeatletter
\@ifpackageloaded{hyperref}{}{\usepackage{hyperref}} 
\@ifpackageloaded{url}{}{\usepackage{url}}           
\makeatother 
\usepackage{booktabs}       
\usepackage{array}
\usepackage{caption}
\usepackage{amsmath}
\usepackage{amssymb}
\usepackage{amsthm}
\usepackage{amsfonts}       
\usepackage{nicefrac}       
\usepackage{microtype}      
\usepackage{xcolor}         
\usepackage{float}         
\usepackage{wrapfig}
\usepackage{pgfplots}
\pgfplotsset{compat=1.18}
\usepgfplotslibrary{groupplots}

\newcommand{\condvspace}[1]{}

\makeatletter
\@ifundefined{c@theorem}{\newtheorem{theorem}{Theorem}[section]}{}
\@ifundefined{c@lemma}{\newtheorem{lemma}[theorem]{Lemma}}{}
\@ifundefined{c@proposition}{\newtheorem{proposition}[theorem]{Proposition}}{}
\@ifundefined{c@corollary}{}{}
\theoremstyle{definition}
\@ifundefined{c@assumption}{\newtheorem{assumption}[theorem]{Assumption}}{}
\@ifundefined{c@construction}{}{}
\makeatother

\title{\textsc{Precise}: SDE-Consistent Stochastic Sampling for RL Post-Training of Flow-Matching Models} 

\author{%
  Jade Zou$^{1,2,*,\ddagger}$, Tao Huang$^{2,*,\ddagger}$, Weijie Kong$^{2,*}$, Junzhe Li$^{1,2,\ddagger}$, Yue Wu$^{2}$, Qi Tian$^{2}$, \\
  Jiangfeng Xiong$^{2}$, Jianwei Zhang$^{2,\dagger}$, Liefeng Bo$^{2}$, Zhao Zhong$^{2,\S}$\\[2pt]
  $^{1}$Peking University \quad $^{2}$Tencent Hunyuan\\
  $^{*}$ Equal Contribution, $^{\S}$ Corresponding Author, $^{\dagger}$ Project Leader\\
  $^{\ddagger}$ Work done during internship at Tencent Hunyuan
}

\begin{document}

\maketitle

\newcommand{\ignore}[1]{}

\begin{abstract}
    Reinforcement learning (RL) has become an effective way to improve prompt alignment and perceptual quality in diffusion and flow-matching generators. A critical step for applying online RL to flow matching is turning the deterministic sampling trajectory into a stochastic policy, typically by replacing the reverse-time Ordinary Differential Equation (ODE) with a Stochastic Differential Equation (SDE). The stochastic sampler, controlling the exploration behavior and denoising dynamics, is thus part of the policy, and its design can significantly affect the reward optimization performance. We break down the sampler design into two interdependent components: choosing the right amount of stochastic exploration, and discretizing the resulting SDE faithfully at the small step counts used in RL. To address the first component, we analyze the inherent tension between exploration and stability in denoising and derive an SDE schedule that balances the two. Turning to the discretization challenge, we use a toy example to show that existing samplers can deviate from the flow-matching process, either by introducing excessive discretization noise or by relying on heuristic rules that do not guarantee convergence to the data distribution. To address these issues, we propose \textsc{Precise}, a new stochastic sampler that balances effective exploration with stability. Crucially, \textsc{Precise} keeps the denoising trajectory SDE-consistent through a novel approximation that freezes the clean-latent posterior mean, resolving the excess noise issue in standard samplers.
    Extensive experiments demonstrate that this formulation leads to significantly faster reward optimization via reinforcement learning, achieving state-of-the-art scores on the optimized rewards (e.g., PickScore, HPSv2.1) while requiring 13.1--53.2\% less wall-clock training time to match the best in-domain performance of prior samplers.
    
    Code is open-source at \url{https://github.com/Tencent-Hunyuan/Precise}.
    
\end{abstract}

\section{Introduction}

Diffusion and flow-matching models are central to visual generation. Diffusion models and score-based SDEs sample by reversing a gradual noising process~\citep{ho2020denoising,song2020score}, while flow matching and rectified flow cast generation as continuous transport~\citep{lipman2022flow,liu2022flow}. Although these pretrained generators provide strong visual priors, they do not directly optimize downstream criteria such as prompt following, perceptual quality, and task-specific constraints. RL post-training addresses this gap by optimizing image-level rewards or preferences via reward supervision, preference optimization, and policy gradients~\citep{xu2023imagereward,wallace2024diffusion,black2023training,fan2023dpok}, with recent extensions to online RL for flow-matching and visual generation models~\citep{liu2025flow,xue2025dancegrpo,zheng2026manifold}.

In online RL, the stochastic sampler is not merely an inference-time choice. It defines the policy whose action probabilities enter the policy-gradient objective~\citep{liu2025flow,xue2025dancegrpo}, determines the exploration available for within-prompt advantage estimation, and controls the denoising trajectory on which the model makes predictions. Thus, for flow-matching models, converting the deterministic ODE into a stochastic sampler creates a sampler-design problem: the sampler must explore enough to support reward optimization while remaining faithful to the denoising dynamics learned during pretraining. This balance is especially delicate because online RL for diffusion and flow-matching models typically restricts the rollout to a small number of denoising steps (e.g., 6-30) to mitigate the massive computational bottleneck of iterative generation. Maintaining faithful trajectory exploration in this extreme low-NFE regime is not merely an efficiency trick, but a fundamental prerequisite for scaling RLHF to high-dimensional continuous domains~\citep{liu2025flow,xue2025dancegrpo}.

To systematically resolve this tension and enable stable exploration under restricted budgets, we deconstruct the stochastic sampler design into two interdependent components. First, the exploration schedule must balance diversity and stability: if the injected noise is too small, sampled groups have limited reward variation and policy improvement is slow; if it is too large, the trajectory can become unstable and reward estimates become less reliable. Second, the finite-step discretization must remain consistent with the intended reverse SDE at a small step count, which becomes harder as the injected noise increases. If the discretization is inconsistent, the model can be forced to extrapolate beyond its training distribution, leading to poor sample quality.

Viewed through the lens of these two requirements, existing samplers expose critical shortcomings. For instance, the Euler-style sampler
used in prior flow-matching RL pipelines~\citep{liu2025flow,xue2025dancegrpo} 
introduces stochasticity, but can add
excess noise under coarse discretization. CPS~\citep{wang2025coefficients} reduces this
artifact by preserving the nominal signal and noise coefficients, but
coefficient preservation alone does not ensure convergence to the correct data distribution,
and the exploration-stability tradeoff still requires careful tuning.
As shown in Figure~\ref{fig:sampler_design}, existing samplers occupy different tradeoff points in the design space.

\begin{figure}[t]
  \centering
  \begin{minipage}[c]{0.40\linewidth}
    \centering
    \includegraphics[height=1.75in]{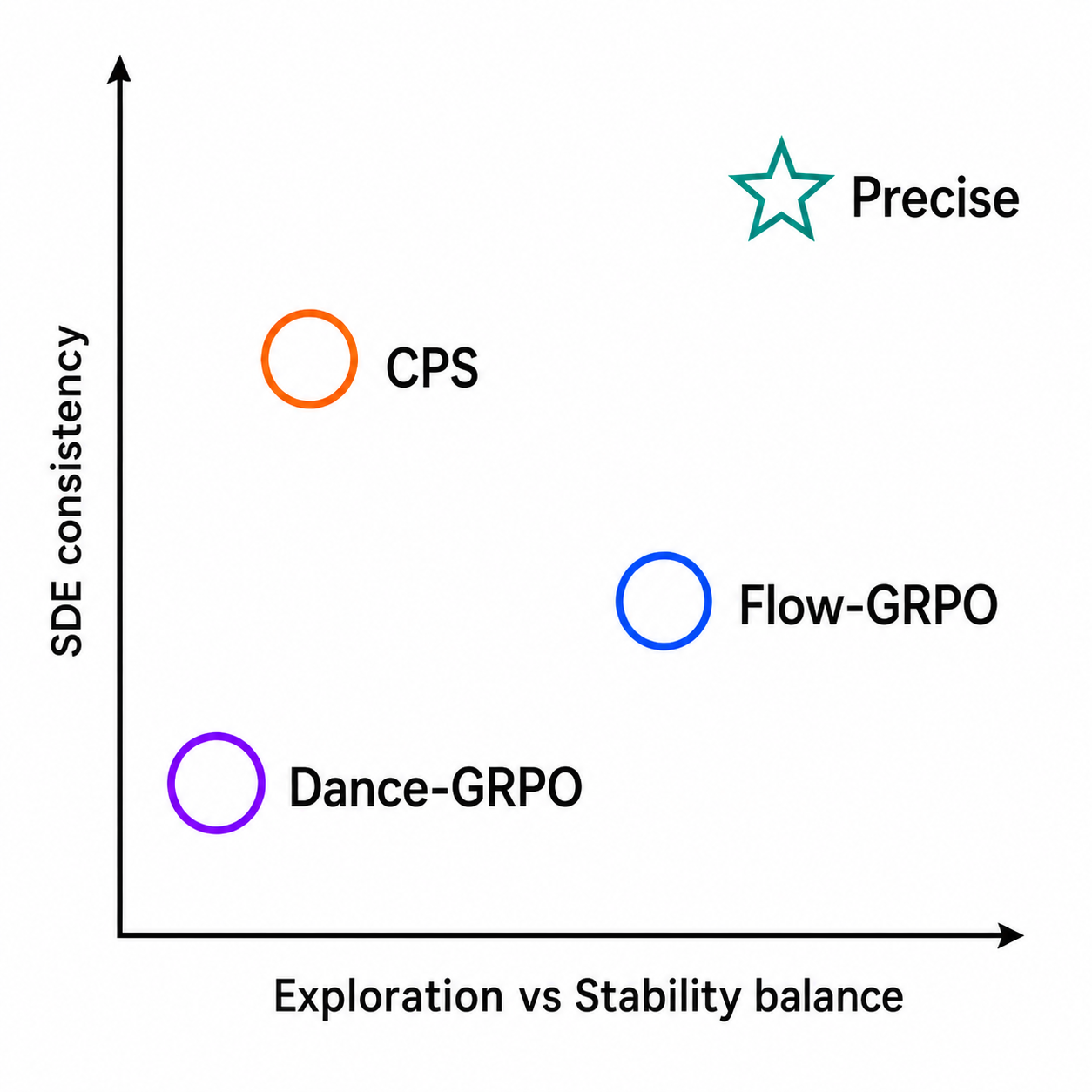}
  \end{minipage}
  \hfill
  \begin{minipage}[c]{0.56\linewidth}
    \centering
    \includegraphics[height=1.75in]{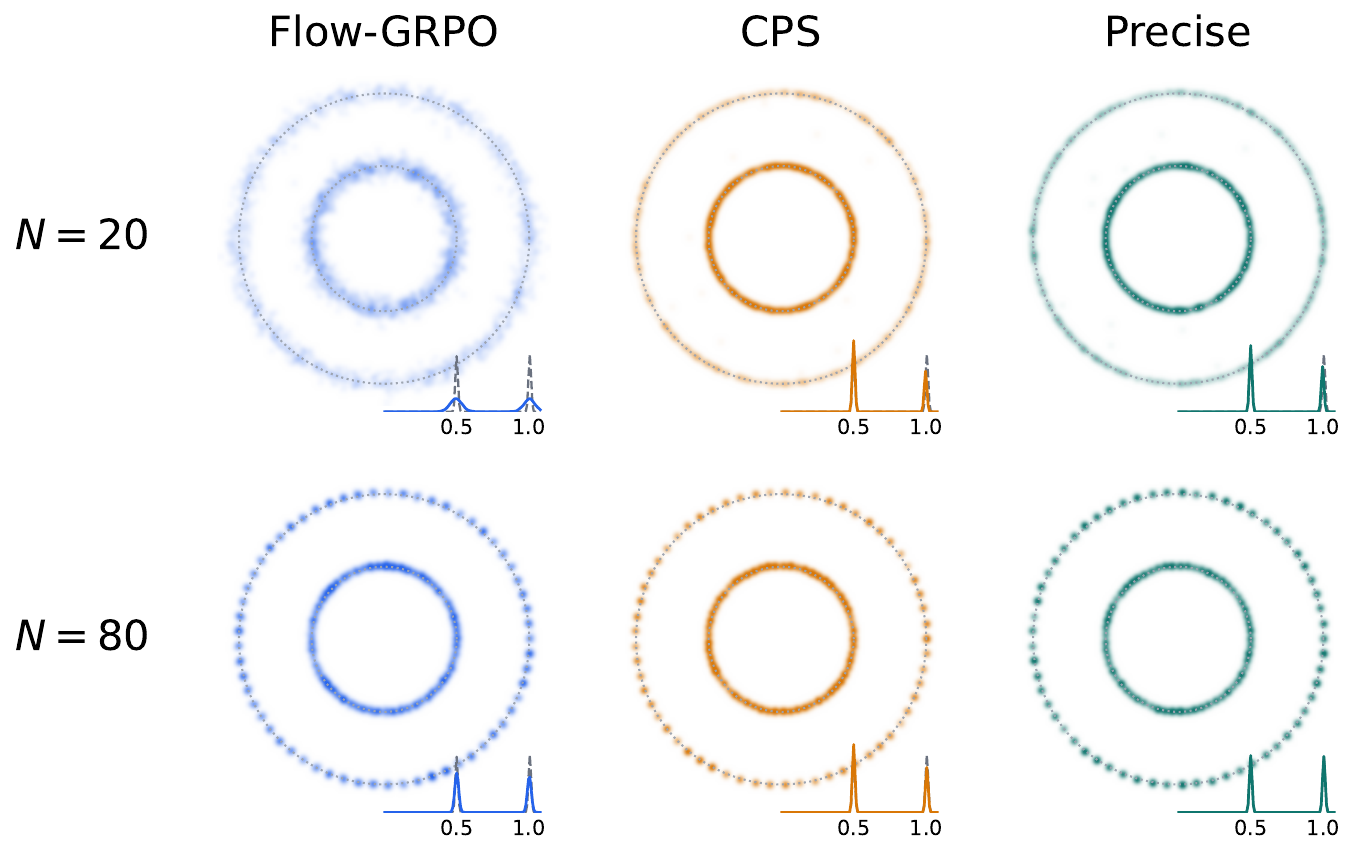}
  \end{minipage}
  \caption{Left (conceptual schematic): Sampler design has two coupled axes: the exploration-stability balance and SDE consistency. Existing stochastic samplers occupy different trade-offs, while \textsc{Precise} improves both. Right: The double-ring example illustrates failure modes of existing samplers: Flow-GRPO~\citep{liu2025flow} introduces excess noise, while CPS~\citep{wang2025coefficients} biases the marginal distribution.}
  \condvspace{-0.1in}
  \label{fig:sampler_design}
\end{figure}

To transcend these suboptimal tradeoffs, we propose \textsc{Precise}, a principled stochastic sampler explicitly designed for online RL post-training. Our framework resolves the tension through two key observations. First, exploration should respect the local geometry of the denoising process. A log signal-to-noise ratio (logSNR) analysis gives a principled time-varying schedule: inject more noise where the model can reliably return to the data manifold, and less where extra randomness would destabilize the trajectory. Second, the finite-step transition should account for the fact that injected noise immediately changes the latent state on which the velocity and score are predicted. Euler-style samplers freeze these quantities within a step, therefore missing this feedback. We instead observe that the clean-latent posterior mean $\hat{\mathbf{z}}_0(t)=\mathbb{E}[\mathbf{z}_0\mid\mathbf{z}_t]$ is more stable under noise injection, because estimating $\hat{\mathbf{z}}_0(t)$ is naturally a denoising task. This motivates freezing $\hat{\mathbf{z}}_0(t)$ rather than the velocity or score, and computing the exact transition under this approximation. The resulting transition rule remains faithful to the intended SDE while avoiding excess discretization noise.

We empirically evaluate \textsc{Precise} on Stable Diffusion 3.5 Medium (SD3.5-M)~\citep{esser2024scaling} and FLUX.2 Klein 4B Base~\citep{blackforestlabs2025flux2,blackforestlabs2026flux2klein}. \textsc{Precise} achieves state-of-the-art scores on the optimized rewards (e.g., PickScore, HPSv2.1) while requiring 13.1--53.2\% less wall-clock training time to match the best in-domain performance of prior samplers, and remains competitive on held-out metrics. Additional ablations show that it is robust to hyperparameter and NFE changes, and that both the exploration schedule and the novel approximation contribute to its performance. Code is open-source at \url{https://github.com/Tencent-Hunyuan/Precise}, including the \textsc{Precise} sampler and reproduction scripts.

Our contributions are summarized as follows:
\begingroup
\setlength{\leftmargini}{1.4em}
\setlength{\labelsep}{0.45em}
\begin{itemize}
\setlength{\itemsep}{0.15em}
\setlength{\topsep}{0.2em}
\setlength{\parsep}{0pt}



    \item We identify stochastic sampler design in RL as a fundamentally coupled problem of exploration scheduling and few-step reverse-SDE discretization, establishing it as a critical component of the RL policy.
    \item We analyze the systematic biases of prior samplers. Building on CPS~\citep{wang2025coefficients}, which showed that Euler-style transitions inject excessive noise at small step counts, we prove that coefficient-preserving (CPS) rules suffer from covariance contraction that biases the marginal distribution.
    \item We introduce \textsc{Precise}, featuring a theoretically principled logSNR-derived exploration schedule and an SDE-consistent closed-form transition. Under a novel ``frozen clean-latent posterior mean'' approximation, it eliminates excess discretization noise.
    \item Extensive experiments demonstrate that \textsc{Precise} achieves state-of-the-art scores on the optimized rewards (e.g., PickScore, HPSv2.1). Crucially, it matches the best prior baselines while requiring 13.1–53.2\% less wall-clock training time, remaining highly robust even in extreme low-budget regimes.
    
\end{itemize}
\endgroup

\section{Related Work}
\begingroup
\setlength{\parskip}{2pt}
\newcommand{\rwpara}[1]{\noindent\textbf{#1}}

\rwpara{Diffusion and flow-matching models.}
Modern visual generators use denoising diffusion and score-SDE formulations~\citep{ho2020denoising,song2020score} or continuous-transport formulations such as flow matching and rectified flow~\citep{lipman2022flow,liu2022flow}. These frameworks underpin many current image and video systems~\citep{esser2024scaling,blackforestlabs2025flux2,wan2025,gao2025seedance,cao2025hunyuanimage,kong2024hunyuanvideo,wu2025hunyuanvideo}.

\rwpara{RL post-training for diffusion and flow-matching models.}
RL post-training aligns visual generators to downstream rewards or preferences via methods such as DDPO, DPOK, ImageReward/ReFL, Diffusion-DPO, and D3PO~\citep{black2023training,fan2023dpok,xu2023imagereward,wallace2024diffusion,yang2024using}. Recent work extends these ideas to ODE-based generators through Dance-GRPO and Flow-GRPO~\citep{xue2025dancegrpo,liu2025flow}; follow-ups improve rollout efficiency~\citep{li2025mixgrpo,li2025branchgrpo,ding2025treegrpo} and optimization stability~\citep{wang2025grpo}. CPS~\citep{wang2025coefficients} is closest to ours: it targets the stochastic sampler used by Flow-GRPO and Dance-GRPO, replacing the Euler transition with a coefficient-preserving rule. Our work further improves the sampler while staying agnostic to the RL objective or training wrapper. Another branch avoids reverse-SDE rollout likelihoods, through forward-process optimization~\citep{zheng2025diffusionnft}, advantage- or reward-weighted matching losses~\citep{xue2025advantage,fan2025online}, or ELBO-based likelihoods computed from the final sample~\citep{choi2026rethinking}; \textsc{Precise} targets the GRPO-style setting, where the sampler defines the policy and its likelihoods.

\rwpara{Few-step sampler design.}
Another line of work studies accurate few-step discretization. Fixed-dynamics solvers such as DDIM, DEIS, DPM-Solver, and UniPC reduce inference error without retraining~\citep{song2020denoising,zhang2022fast,lu2022dpm,zhao2023unipc}, while accelerators learn or distill trajectories into very small NFE~\citep{salimans2022progressive,song2023consistency,luo2023latent}, for example by choosing training couplings that straighten the learned flow~\citep{lin2026beyond} or by conditioning on sub-modes to keep one-step generation diverse~\citep{lin2026subflow}. \textsc{Precise} is closest in spirit to fixed-dynamics samplers, but targets a different regime: in RL post-training, stochasticity is the exploration mechanism rather than a nuisance, so the transition must be both few-step accurate and stochastic inside policy optimization.
\par
\endgroup

\section{Preliminaries}

\subsection{Flow matching and reverse-time sampling}

We write a general diffusion or interpolation path as
$\mathbf{z}_t=\alpha_t\mathbf{z}_0+\sigma_t\boldsymbol{\epsilon}$, where
$\boldsymbol{\epsilon}\sim\mathcal{N}(\mathbf{0},\mathbf{I})$, $t\in[0,1]$
indexes the noise level, and larger $t$ corresponds to a noisier state. In the
flow-matching parameterization, $\alpha_t=1-t$ and $\sigma_t=t$:
\begin{equation}
\mathbf{z}_t = (1-t)\mathbf{z}_0 + t\boldsymbol{\epsilon}.
\label{eq:fm_path}
\end{equation}
The target velocity is
$\mathbf{u}_t=\boldsymbol{\epsilon}-\mathbf{z}_0$. A flow-matching model learns
to predict $\mathbf{u}_t$ from
$(\mathbf{z}_t, t)$. At test time, sampling integrates the
probability-flow ODE backward from $t=1$ to $t=0$.

For later use, define the clean-latent posterior mean as
$\hat{\mathbf{z}}_0(t)\triangleq\mathbb{E}[\mathbf{z}_0\mid\mathbf{z}_t]$.
Under the linear path in Eq.~\eqref{eq:fm_path}, the posterior
mean of the noise and the score are
\begin{align}
\hat{\boldsymbol{\epsilon}}(t)
&=
\frac{\mathbf{z}_t - (1-t)\hat{\mathbf{z}}_0(t)}{t},
\label{eq:posterior_mean_eps}
\\
\nabla_{\mathbf{z}}\log p_t(\mathbf{z}_t)
&=
-\frac{\hat{\boldsymbol{\epsilon}}(t)}{t}
=
\frac{(1-t)\hat{\mathbf{z}}_0(t)-\mathbf{z}_t}{t^2}.
\label{eq:score_posterior_mean}
\end{align}
This follows by differentiating the Gaussian conditional density
$p(\mathbf{z}_t\mid \mathbf{z}_0)$ with respect to $\mathbf{z}_t$ and taking
the posterior expectation over $\mathbf{z}_0$.

\subsection{RL post-training for flow matching}

RL post-training requires stochastic rollouts. Flow-GRPO~\citep{liu2025flow} and
Dance-GRPO~\citep{xue2025dancegrpo} therefore replace the deterministic reverse-time
ODE with the reverse-time SDE
\begin{equation}
\mathrm{d}\mathbf{z}_t
=
\Big(
\mathbf{u}_t
- \tfrac{1}{2}\varepsilon_t^2 \nabla_{\mathbf{z}} \log p_t(\mathbf{z}_t)
\Big)\,\mathrm{d}t
+ \varepsilon_t \,\mathrm{d}\mathbf{w}_t,
\label{eq:fm_sde}
\end{equation}
where $\varepsilon_t \ge 0$ controls exploration and
$\mathbf{w}_t$ is Brownian motion. In ideal continuous time with exact scores,
this SDE preserves the flow-matching marginals for any $\varepsilon_t$. In
practice, $\varepsilon_t$ and the finite-step discretization of
Eq.~\eqref{eq:fm_sde} both matter because training uses imperfect predictions
and a small NFE budget.

Given a stochastic sampler, recent work optimizes groups of trajectories
$\{\tau_i\}_{i=1}^{G}$ with a GRPO-style clipped policy objective and
group-normalized rewards~\citep{shao2024deepseekmath}.

\section{Method and Analysis}
\label{sec:analysis}

The stochastic sampler used for RL post-training has two coupled design axes:
the exploration schedule $\varepsilon_t$ and the finite-step transition rule.
The schedule determines how much randomness the policy injects at each noise
level, while the transition rule determines whether the resulting reverse SDE
is faithfully realized at small training NFE. We first derive a schedule that
balances exploration and denoising stability. We then show how existing
samplers deviate from the flow-matching process, and derive \textsc{Precise},
an SDE-consistent closed-form transition under a local approximation that freezes the clean-latent posterior
mean.

\subsection{Balancing exploration and stability}

We start with the exploration schedule $\varepsilon_t$. As shown in
Eq.~\eqref{eq:fm_sde}, $\varepsilon_t$ controls the strength of the injected
noise while also scaling the score term. We choose $\varepsilon_t$ to balance
exploration and reliability: large enough for diversity, yet stable enough to
keep model error controlled along the denoising trajectory.

To formalize this tradeoff, we analyze how much denoising progress is made by
the velocity term $\mathbf{u}_t \mathrm{d}t$ and the score term
$- \tfrac{1}{2}\varepsilon_t^2 \nabla_{\mathbf{z}} \log p_t(\mathbf{z}_t)
\mathrm{d}t$, and how much is canceled by the stochastic term
$\varepsilon_t \,\mathrm{d}\mathbf{w}_t$. We use log signal-to-noise ratio
(logSNR) to measure denoising progress, following prior SNR and
noise-level parameterizations~\citep{kingma2021variational,karras2022elucidating}.
For $\mathbf{z}_t=\alpha_t\mathbf{z}_0+\sigma_t\boldsymbol{\epsilon}$, define
$\mathrm{SNR}(t)\triangleq\alpha_t^2/\sigma_t^2$ and
$\lambda_t\triangleq\log\mathrm{SNR}(t)$; under flow matching,
$(\alpha_t,\sigma_t)=(1-t,t)$, so $\lambda_t=\log((1-t)^2/t^2)$.

Under this progress measure, we can analyze the first-order contribution of
each term to the logSNR change independently. Consider the reverse SDE in
Eq.~\eqref{eq:fm_sde}. Over an infinitesimal reverse step from $t$ to
$t-\Delta t$, there are the velocity term $-\mathbf{u}_t\Delta t$,
score term $\frac{1}{2}\varepsilon_t^2\nabla_{\mathbf{z}}\log p_t(\mathbf{z}_t)\Delta t$,
and stochastic term $\varepsilon_t\sqrt{\Delta t}\,\mathbf{w}$, with
$\mathbf{w}\sim\mathcal{N}(\mathbf{0},\mathbf{I})$.
The first-order logSNR contributions of the three terms are as follows.

\begin{lemma}[First-order logSNR decomposition]
\label{lem:logsnr_decomp}
For the flow-matching reverse SDE, as $\Delta t \to 0$,
\begin{equation*}
\Delta\lambda_{\mathrm{vel}}=\frac{2\Delta t}{t(1-t)}+o(\Delta t),\quad
\Delta\lambda_{\mathrm{sco}}=\frac{\varepsilon_t^2\Delta t}{t^2}+o(\Delta t),\quad
\Delta\lambda_{\mathrm{sto}}=-\frac{\varepsilon_t^2\Delta t}{t^2}+o(\Delta t),
\end{equation*}
where the three quantities denote the first-order contributions of the
velocity term, score term, and stochastic term, respectively.
\end{lemma}

Lemma~\ref{lem:logsnr_decomp} shows that the net first-order progress is
produced by the velocity term and is independent of $\varepsilon_t$, while the
intermediate score and noise cancelation scales like $\varepsilon_t^2/t^2$.
The score-term progress is nevertheless not free: before it is canceled by
re-randomization, it asks the model to produce extra denoising progress beyond
the normal ODE velocity progress. If this extra
progress is much larger than the velocity-induced progress, model prediction
errors can be amplified. We therefore control this burden by keeping the ratio
$R(t)\triangleq\Delta\lambda_{\mathrm{sco}}/\Delta\lambda_{\mathrm{vel}}$
constant in time, with its magnitude set by a scalar parameter, giving the schedule
\begin{equation}
\varepsilon_t
=
\eta\sqrt{\frac{t}{1-t}},
\label{eq:eps_schedule}
\end{equation}
where $\eta>0$ sets the magnitude of the extra model-driven progress. Under
this schedule,
$R(t)=(\varepsilon_t^2/t^2)/(2/(t(1-t)))=\eta^2/2$.
Thus $\eta$ directly controls the score-to-velocity progress ratio; in
practice, values around $\sqrt{2}$ work well across tasks, where the
score-induced progress is comparable to the normal velocity progress.
A constant $R(t)$ keeps this burden at a fixed proportion of the velocity
progress near both endpoints and leaves one scalar to tune. Flow-GRPO's schedule
and the memoryless schedule of Adjoint Matching share the factor
$\sqrt{t/(1-t)}$ and differ only in $\eta$: $\eta=0.7\approx\sqrt{2}/2$~\citep{liu2025flow}
and $\eta=\sqrt{2}$~\citep{domingo2025adjoint}.

\subsection{Existing samplers deviate from the flow-matching process}

After fixing the exploration schedule, and therefore the denoising SDE, the
remaining design problem is the finite-step sampler. It should follow the
reverse SDE closely enough that reward optimization sees a faithful stochastic
policy with minimal discretization artifact.

\paragraph{Euler sampler introduces excess discretization noise.}
Consider the standard Euler--Maruyama discretization of the reverse SDE used by
Flow-GRPO~\citep{liu2025flow} and Dance-GRPO~\citep{xue2025dancegrpo}. One
step from $t$ to $t' < t$ is
\begin{equation*}
\mathbf{z}_{t'}^{\mathrm{Euler}}
=
\mathbf{z}_t
+\Big(
-\mathbf{u}_t
+ \tfrac{1}{2}\varepsilon_t^2 \nabla_{\mathbf{z}}\log p_t(\mathbf{z}_t)
\Big)\Delta t
+ \varepsilon_t \sqrt{\Delta t}\,\mathbf{w},
\end{equation*}
with $\Delta t = t - t'$. This discretization freezes the velocity and score at
time $t$. That approximation is mostly harmless for the
deterministic ODE, but problematic once new noise is injected: the noise changes
both the velocity and score inside the step, while Euler transition ignores that feedback.

CPS~\citep{wang2025coefficients} identified this excess noise through a
one-step coefficient analysis; a point-mass toy example, for which that
analysis is exact, exposes the mismatch. Assume the data distribution is a point mass at
$\mathbf{z}_\star$. Solving Eq.~\eqref{eq:fm_path} for
$\boldsymbol{\epsilon}$ and plugging the resulting velocity and score into the
Euler step gives
\begin{equation*}
\mathbf{z}_{t'}^{\mathrm{Euler}}
=(1-t')\mathbf{z}_\star
+ \left(t' - \frac{\varepsilon_t^2 \Delta t}{2t}\right)\boldsymbol{\epsilon}
+ \varepsilon_t \sqrt{\Delta t}\,\mathbf{w}.
\end{equation*}
Because $\mathbf{w}$ is independent of $\boldsymbol{\epsilon}$, the noise
coefficient becomes
$\sqrt{(t' - \varepsilon_t^2 \Delta t/(2t))^2+\varepsilon_t^2 \Delta t}$,
which is strictly larger than $t'$ whenever $\varepsilon_t > 0$. Thus Euler
adds excess noise even with a perfect model and the simplest data law.

\paragraph{CPS preserves coefficients while collapsing residual uncertainty.}
CPS~\citep{wang2025coefficients} proposes transition rules of the form
\begin{equation*}
\mathbf{z}_{t'}^{\mathrm{CPS}}
=(1-t')\hat{\mathbf{z}}_0(t)
+ k_1 \hat{\boldsymbol{\epsilon}}(t)
+ k_2 \mathbf{w},
\qquad
k_1^2 + k_2^2 = t'^2.
\end{equation*}
This coefficient-preserving rule is necessary for exactness under the
point-mass model. For non-degenerate data distributions, coefficient
preservation leaves the marginal law underdetermined because posterior means
discard residual uncertainty in $\mathbf{z}_0\mid\mathbf{z}_t$.
Specifically, the deterministic part of the CPS transition is the posterior
mean of a coefficient-preserved random variable. Total covariance then shows
that CPS contracts the covariance whenever that variable still has
posterior uncertainty after conditioning on $\mathbf{z}_t$.
For a simple local case, take $t'\approx t$ with $k_1=0$ and $k_2=t'$.
The transition becomes
$\mathbf{z}_{t'}^{\mathrm{CPS}}=(1-t')\mathbb{E}[\mathbf{z}_0\mid \mathbf{z}_t]+t'\mathbf{w}$,
whereas the corresponding flow-matching marginal would use
$(1-t')\mathbf{z}_0+t'\boldsymbol{\epsilon}$. The covariance contraction,
defined as the target covariance minus the CPS covariance, is exactly
$(1-t')^2\mathbb{E}_{\mathbf{z}_t}[\operatorname{Cov}(\mathbf{z}_0\mid\mathbf{z}_t)]$,
so any remaining ambiguity about the clean latent is removed by the posterior
mean. Appendix~\ref{app:cps_variance} gives the general derivation and a
double-ring example where this contraction appears as a persistent marginal
bias.

\subsection{\textsc{Precise}: an SDE-consistent finite-step transition}

The finite-step transition should preserve reverse-SDE dynamics while
remaining usable at small NFE budgets required by RL. The previous two
cases isolate two constraints: Euler freezing ignores the feedback from
freshly injected noise, while coefficient preservation alone can bias the sample distribution. We therefore retain the SDE but seek a local
anchor robust to injected noise.

The natural choice is the clean-latent posterior mean given the current state,
$\hat{\mathbf{z}}_0(t)
=\mathbb{E}[\mathbf{z}_0\mid\mathbf{z}_t]$. 
Since estimating this posterior mean is naturally a denoising task, this anchor
generally moves much more slowly than the noisy latent under isotropic
perturbations. Unlike CPS, we use it only to
anchor the linearized reverse SDE; the residual stochastic
variance is still set by the SDE transition. 

Using Eqs.~\eqref{eq:posterior_mean_eps} and
\eqref{eq:score_posterior_mean}, the reverse SDE can be rewritten as
\begin{equation}
\mathrm{d}\mathbf{z}_t
=
\left(
\frac{\mathbf{z}_t-\hat{\mathbf{z}}_0(t)}{t}
+ \frac{\varepsilon_t^2}{2t^2}
\bigl(
\mathbf{z}_t-(1-t)\hat{\mathbf{z}}_0(t)
\bigr)
\right)\mathrm{d}t
+ \varepsilon_t\,\mathrm{d}\mathbf{w}_t.
\label{eq:fm_sde_mean_form}
\end{equation}

\begin{assumption}[Frozen posterior mean on one step]
\label{ass:frozen_mean}
For a reverse step from $t$ to $t' < t$, we treat the clean-latent posterior
mean along the step as approximately constant, i.e.,
$\hat{\mathbf{z}}_0(s)\approx\hat{\mathbf{z}}_0(t)$ for all $s\in[t',t]$.
\end{assumption}

Assumption~\ref{ass:frozen_mean} is exact for point-mass data and serves as a
local approximation over one step. Under this assumption, the SDE becomes
linear and admits the following closed-form transition, which is exact for the
frozen-mean SDE; Section~\ref{sec:frozen_mean_discussion} measures how closely
it tracks the original reverse SDE.

\begin{theorem}[Exact transition under frozen posterior mean]
\label{cons:exact_update}
Assume Assumption~\ref{ass:frozen_mean}. Then the exact transition from
$\mathbf{z}_t$ to $\mathbf{z}_{t'}$ is
\begin{equation*}
\mathbf{z}_{t'}
=(1-t')\hat{\mathbf{z}}_0(t)
+ \frac{t'}{t}e^{-A(t',t)/2}
\bigl(
\mathbf{z}_t-(1-t)\hat{\mathbf{z}}_0(t)
\bigr)
+ t'\sqrt{1-e^{-A(t',t)}}\,\mathbf{w},
\end{equation*}
where
$A(t',t)\triangleq \int_{t'}^t \varepsilon_s^2/s^2\,\mathrm{d}s$ and
$\mathbf{w}\sim\mathcal{N}(\mathbf{0}, \mathbf{I})$ is independent of
$\mathbf{z}_t$.
\end{theorem}

The proof is given in Appendix~\ref{app:proof_exact_update}. In practice, we
replace the exact $\hat{\mathbf{z}}_0(t)$ by the model clean-latent prediction
and evaluate $A(t',t)$ analytically or numerically, depending on
$\varepsilon_t$.

For the exploration schedule in Eq.~\eqref{eq:eps_schedule},
$A(t',t)=\eta^2\int_{t'}^t 1/(s(1-s))\,\mathrm{d}s
=\eta^2\log(t(1-t')/(t'(1-t)))$, so Eq.~\eqref{eq:posterior_mean_eps} turns
Theorem~\ref{cons:exact_update} into the explicit transition rule
\begin{equation*}
\mathbf{z}_{t'}
=(1-t')\hat{\mathbf{z}}_0(t)
+ t'\rho(t',t)\hat{\boldsymbol{\epsilon}}(t)
+ t'\sqrt{1-\rho(t',t)^2}\,\mathbf{w},
\end{equation*}
with
\begin{equation}
\rho(t',t)
\triangleq
\left(
\frac{t'(1-t)}{t(1-t')}
\right)^{\eta^2/2},
\label{eq:precise_closed_form}
\end{equation}
obtaining the final transition rule used by \textsc{Precise}.

\subsection{Discussion}
\label{sec:frozen_mean_discussion}

\begin{figure}[t]
  \centering
  \begin{minipage}[c]{0.44\linewidth}
    \centering
    \includegraphics[width=\linewidth]{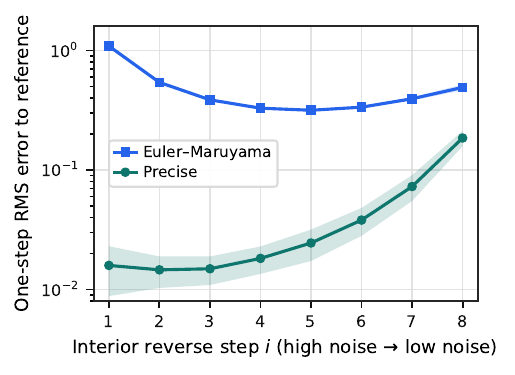}
  \end{minipage}\hfill
  \begin{minipage}[c]{0.52\linewidth}
    \caption{One-step error to a fine-grid reverse-SDE reference on SD3.5-M (10 NFE, $\eta=1.5$). Euler--Maruyama freezes the velocity and score within a step, while \textsc{Precise} freezes the clean-latent prediction; bands show $\pm$1 std over eight trajectories.}
    \label{fig:frozen_mean_diagnostics}
  \end{minipage}
\end{figure}

The main approximation in \textsc{Precise} is local: during one finite reverse
step, we treat the clean-latent posterior mean as fixed.
Appendix~\ref{app:proof_frozen_mean_local} bounds the resulting one-step
transition error by the drift of this mean within the step
(Proposition~\ref{prop:frozen_mean_error}); the posterior covariance controls
its sensitivity to the state (Lemma~\ref{lem:mean_jacobian}).

We measure both quantities directly on SD3.5-M~\citep{esser2024scaling},
against a reference that integrates the reverse SDE with 400
Euler--Maruyama substeps per step; the one-step samplers are driven by the
reference's fine-grid Brownian increments
(Appendix~\ref{app:frozen_mean_diagnostics}). Over the eight interior steps,
the clean-latent prediction drifts by $0.21$ of the step displacement on
average over each step, so freezing it is an approximation rather than an
identity. Even so, one \textsc{Precise} step lands far closer to the reference
than one Euler--Maruyama step on the same SDE, which freezes the velocity and
score instead (Fig.~\ref{fig:frozen_mean_diagnostics}): the RMS error is
$0.048$ versus $0.485$ on average, and \textsc{Precise} is closer on all $64$
trajectory--step pairs, although the error ratio rises from $0.015$ to $0.38$
toward the data end.

The closure is less accurate where the predicted clean latent changes quickly
relative to the noisy state: near ambiguous or multimodal states with large
posterior covariance (Lemma~\ref{lem:mean_jacobian}), where the clean-latent
prediction is inaccurate, and for steps large relative to the noise level,
such as the last coarse steps, where the prediction moves almost one-for-one
with the noisy state.

\section{Experiments}
\label{sec:experiments}

We evaluate the sampler induced by Theorem~\ref{cons:exact_update}, abbreviated as \textsc{Precise}, against Dance-GRPO~\citep{xue2025dancegrpo}, Flow-GRPO~\citep{liu2025flow}, and CPS~\citep{wang2025coefficients} for GRPO-based reward optimization of flow-matching models. 

\subsection{Experimental setup}

\paragraph{Rewards.} For rule-based optimization, we use GenEval \citep{ghosh2023geneval}. For model-based rewards, we train with PickScore~\citep{kirstain2023pick}, CLIPScore~\citep{hessel2021clipscore}, and HPSv2.1~\citep{wu2023human}, and report ImageReward~\citep{xu2023imagereward}, UnifiedReward~\citep{wang2025unified}, and Aesthetics~\citep{laion2022aestheticpredictor} at evaluation. These metrics cover complementary axes of reward alignment, text-image matching, and perceptual quality.

\paragraph{Prompt datasets.} GenEval prompts are synthetic object-centric prompts from task templates over COCO categories~\citep{ghosh2023geneval}. PickScore, CLIPScore, and HPSv2.1 prompts come from Pick-a-Pic training captions~\citep{kirstain2023pick}, filtered to unique captions with at least six whitespace-separated words and split by holding out 2{,}048 shuffled captions for evaluation. In all cases, we use the prompt splits distributed with Flow-GRPO codebase~\citep{liu2025flow}.

\paragraph{Training and evaluation protocol.} We base our experiments on the Flow-GRPO codebase~\citep{liu2025flow}. Our main backbone is Stable Diffusion 3.5 Medium (SD3.5-M) \citep{esser2024scaling}, a pretrained text-to-image flow-matching model, at $512\times512$ resolution. We fine-tune with LoRA ($r=32$, $\alpha=64$) and use GRPO-Guard~\citep{wang2025grpo}. We use no classifier-free guidance during training or evaluation and no KL loss during training. We use the same NFE for training and evaluation. Unless stated otherwise, results use the final checkpoint and each method's default scalar exploration parameter $\eta$: $0.3$ for Dance-GRPO~\citep{xue2025dancegrpo}, $0.7$ for Flow-GRPO~\citep{liu2025flow}, $0.7$ for CPS~\citep{wang2025coefficients}, and $1.5$ for \textsc{Precise}. We keep the rest of the Flow-GRPO recipe and dataset splits fixed, changing only the sampler and exploration setting. All runs use 8$\times$ NVIDIA H20 GPUs, taking about 80GB GPU memory per H20.

\subsection{Main results}
\label{sec:sd35m_results}

We train all four samplers on SD3.5-M under two protocols: 10-NFE/3000-iteration and 30-NFE/1000-iteration, both taking about 3 days. We report final-checkpoint training rewards and held-out metrics,\footnote{For Dance-GRPO in the 30-NFE protocol, we report iteration 600, the highest checkpoint before training becomes unstable.} using mean and standard deviation over three evaluation seeds where available, and plot training curves for the three in-domain rewards. We also include DiffusionNFT~\citep{zheng2025diffusionnft} as a reference, using their default 25-NFE training setting, training for 3 days, and evaluating at 30 NFE.\footnote{We tried to accelerate DiffusionNFT with 10-NFE training, but it is less stable and diverges after two days. The best pre-divergence checkpoint is worse than the default setting.}

\newcommand{\resultpm}[2]{#1{\scriptsize$\pm$#2}}

\begin{table}[htb]
  \centering
  \scriptsize
  \setlength{\tabcolsep}{2pt}
  \caption{Main SD3.5-M results under the 10-NFE/3000-iteration and 30-NFE/1000-iteration protocols. Multi-seed entries are mean$\pm$standard deviation over three evaluation seeds. Higher is better for every metric.}
  \label{tab:main_results_sd35m}
  \begin{tabular}{@{}lcccccc@{}}
    \toprule
    & \multicolumn{3}{c}{Training rewards} & \multicolumn{3}{c}{Held-out metrics} \\
    \cmidrule(lr){2-4}\cmidrule(l){5-7}
    Method & PickScore $\uparrow$ & CLIPScore $\uparrow$ & HPSv2.1 $\uparrow$ & ImageReward $\uparrow$ & UnifiedReward v2 $\uparrow$ & Aesthetics $\uparrow$ \\
    \midrule
    \multicolumn{7}{@{}l}{\textbf{10-NFE protocol}} \\
    Dance-GRPO~\citep{xue2025dancegrpo} & \resultpm{22.428}{0.007} & \resultpm{0.942}{0.000} & \resultpm{0.340}{0.000} & \resultpm{1.350}{0.010} & \resultpm{0.647}{0.000} & \resultpm{6.620}{0.007} \\
    Flow-GRPO~\citep{liu2025flow} & \resultpm{23.242}{0.003} & \resultpm{1.018}{0.001} & \resultpm{0.377}{0.000} & \textbf{\resultpm{1.627}{0.004}} & \underline{\resultpm{0.652}{0.000}} & \resultpm{6.530}{0.003} \\
    CPS~\citep{wang2025coefficients} & \underline{\resultpm{23.670}{0.002}} & \underline{\resultpm{1.026}{0.001}} & \underline{\resultpm{0.389}{0.000}} & \resultpm{1.607}{0.005} & \resultpm{0.650}{0.000} & \textbf{\resultpm{6.667}{0.003}} \\
    \textsc{Precise} & \textbf{\resultpm{23.745}{0.004}} & \textbf{\resultpm{1.038}{0.001}} & \textbf{\resultpm{0.391}{0.000}} & \underline{\resultpm{1.615}{0.004}} & \textbf{\resultpm{0.654}{0.000}} & \underline{\resultpm{6.652}{0.003}} \\
    \midrule
    \multicolumn{7}{@{}l}{\textbf{30-NFE protocol}} \\
    Dance-GRPO~\citep{xue2025dancegrpo} & \resultpm{22.237}{0.004} & \resultpm{0.868}{0.001} & \resultpm{0.330}{0.000} & \resultpm{1.082}{0.008} & \resultpm{0.652}{0.001} & \resultpm{6.568}{0.005} \\
    Flow-GRPO~\citep{liu2025flow} & \resultpm{23.227}{0.006} & \resultpm{0.967}{0.001} & \resultpm{0.374}{0.000} & \resultpm{1.528}{0.002} & \underline{\resultpm{0.656}{0.000}} & \underline{\resultpm{6.736}{0.004}} \\
    CPS~\citep{wang2025coefficients} & \underline{\resultpm{23.288}{0.007}} & \underline{\resultpm{0.994}{0.001}} & \textbf{\resultpm{0.379}{0.000}} & \underline{\resultpm{1.542}{0.001}} & \resultpm{0.652}{0.000} & \resultpm{6.693}{0.002} \\
    \textsc{Precise} & \textbf{\resultpm{23.421}{0.006}} & \textbf{\resultpm{1.002}{0.001}} & \textbf{\resultpm{0.379}{0.000}} & \textbf{\resultpm{1.548}{0.004}} & \textbf{\resultpm{0.658}{0.000}} & \textbf{\resultpm{6.766}{0.003}} \\
    \midrule
    \textcolor{black!50}{DiffusionNFT~\citep{zheng2025diffusionnft}} & \textcolor{black!50}{23.349} & \textcolor{black!50}{1.010} & \textcolor{black!50}{0.363} & \textcolor{black!50}{1.573} & \textcolor{black!50}{0.660} & \textcolor{black!50}{6.494} \\
    \bottomrule
  \end{tabular}
\end{table}

\begin{figure}[htb]
  \condvspace{-0.6em}
  \centering
  \includegraphics[width=0.92\linewidth]{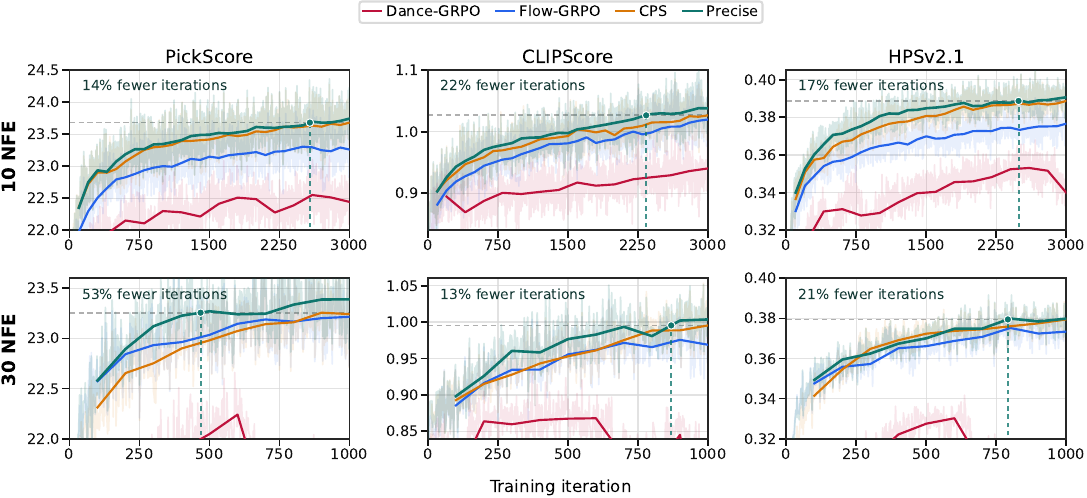}
  \condvspace{-0.8em}
  \caption{Training trajectories under the two main protocols. Higher is better.}
  \label{fig:main_training_curves}
\end{figure}

As shown in Table~\ref{tab:main_results_sd35m} and Figure~\ref{fig:main_training_curves}, \textsc{Precise} ranks first on all three training rewards and remains competitive on held-out metrics, and the same ordering holds through most of optimization. Using the best value among the prior samplers on each metric as the target, \textsc{Precise} reaches the PickScore, CLIPScore, and HPSv2.1 targets with 14.2\%--22.1\% fewer iterations in the 10-NFE protocol, and 13.1\%--53.2\% fewer iterations in the 30-NFE protocol, directly translating to wall-clock savings since sampler overhead is negligible.

The gaps among stochastic samplers highlight both sampler-design components: exploration-stability balance and SDE consistency. Dance-GRPO uses a constant noise schedule $\varepsilon_t=\eta$, too conservative in early denoising and too aggressive in late steps, according to our analysis in Section~\ref{sec:analysis}. It yields little layout variation and corrupts fine details, causing training collapse. A similar, smaller gap exists between \textsc{Precise} and CPS. In contrast, the gap between \textsc{Precise} and Flow-GRPO mainly stems from removing excess discretization noise, which boosts sample quality and reward scores.

Since Dance-GRPO underperforms, we exclude it from later experiments for efficiency.

\subsection{Task-level validation}

As a task-level check, we isolate one rule-based reward and one model-based reward in the 10-NFE/3000-iteration protocol. Table~\ref{tab:head_to_head} shows that \textsc{Precise} is best on both GenEval and PickScore.

\begin{table}[htb]
  \centering
  \small
  \caption{Head-to-head results under the 10-NFE protocol. Higher is better.}
  \label{tab:head_to_head}
  \begin{tabular}{@{}lcc@{}}
    \toprule
    Method & GenEval $\uparrow$ & PickScore $\uparrow$ \\
    \midrule
    Flow-GRPO~\citep{liu2025flow} & 0.969 & 24.283 \\
    CPS~\citep{wang2025coefficients} & \underline{0.972} & \underline{24.594} \\
    \textsc{Precise} & \textbf{0.980} & \textbf{24.668} \\
    \bottomrule
  \end{tabular}
\end{table}

\subsection{Results on FLUX.2}
\label{sec:flux2_klein_results}

We further evaluate on FLUX.2 Klein 4B Base~\citep{blackforestlabs2025flux2,blackforestlabs2026flux2klein} at $512\times512$ resolution. We fine-tune with LoRA ($r=16$, $\alpha=16$) under a 20-NFE/1000-iteration protocol, and train for about 1 day on 8$\times$ NVIDIA H20 GPUs. Table~\ref{tab:flux20_results} and Figure~\ref{fig:flux20_training_curves} report final rewards and training curves.

\begin{table}[htb]
  \centering
  \scriptsize
  \setlength{\tabcolsep}{2pt}
  \caption{FLUX.2 Klein 20-NFE results at 1000 iterations. Entries are mean$\pm$standard deviation over three evaluation seeds. Higher is better for every metric.}
  \label{tab:flux20_results}
  \begin{tabular}{@{}lcccccc@{}}
    \toprule
    & \multicolumn{3}{c}{Training rewards} & \multicolumn{3}{c}{Held-out metrics} \\
    \cmidrule(lr){2-4}\cmidrule(l){5-7}
    Method & PickScore $\uparrow$ & CLIPScore $\uparrow$ & HPSv2.1 $\uparrow$ & ImageReward $\uparrow$ & UnifiedReward v2 $\uparrow$ & Aesthetics $\uparrow$ \\
    \midrule
    Flow-GRPO~\citep{liu2025flow} & \underline{\resultpm{22.591}{0.006}} & \underline{\resultpm{0.905}{0.001}} & \resultpm{0.355}{0.000} & \resultpm{1.266}{0.001} & \textbf{\resultpm{0.654}{0.000}} & \underline{\resultpm{6.676}{0.005}} \\
    CPS~\citep{wang2025coefficients} & \resultpm{22.486}{0.005} & \resultpm{0.898}{0.001} & \underline{\resultpm{0.362}{0.000}} & \underline{\resultpm{1.334}{0.010}} & \resultpm{0.649}{0.001} & \resultpm{6.513}{0.000} \\
    \textsc{Precise} & \textbf{\resultpm{22.791}{0.010}} & \textbf{\resultpm{0.930}{0.001}} & \textbf{\resultpm{0.366}{0.000}} & \textbf{\resultpm{1.410}{0.006}} & \underline{\resultpm{0.653}{0.000}} & \textbf{\resultpm{6.677}{0.007}} \\
    \bottomrule
  \end{tabular}
\end{table}

\begin{figure}[htb]
  \centering
  \includegraphics[width=0.92\linewidth]{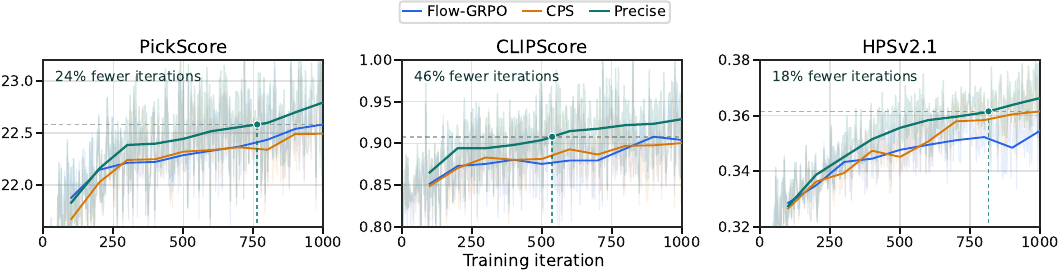}
  \condvspace{-0.7em}
  \caption{FLUX.2 Klein training trajectories under the 20-NFE protocol. Higher is better.}
  \label{fig:flux20_training_curves}
\end{figure}

Table~\ref{tab:flux20_results} shows that gains transfer from SD3.5-M to the newer FLUX.2 backbone. At the final checkpoint, \textsc{Precise} ranks first on five of six metrics and remains within $0.001$ of the best sampler on UnifiedReward v2. Figure~\ref{fig:flux20_training_curves} shows the same trend: \textsc{Precise} rises faster and finishes higher on the three training rewards. Using the best value among the prior samplers as the target, \textsc{Precise} reaches the PickScore, CLIPScore, and HPSv2.1 targets with 18.4\%--46.3\% fewer iterations.

\subsection{Ablations}

\paragraph{Exploration schedule and discretization.}
We isolate the two sampler-design components with a progression from
Dance-GRPO~\citep{xue2025dancegrpo} to \textsc{Precise}. First, we replace the
Euler-style transition in Dance-GRPO with our SDE-consistent finite-step
transition while keeping the constant exploration schedule. 
Second, we replace the constant schedule with the
logSNR-derived schedule in Eq.~\eqref{eq:eps_schedule}, giving
\textsc{Precise}. We set $\eta=0.7$ for the intermediate sampler, about a
$2\times$ increase over default Dance-GRPO, enabled
by the improved stability from our discretization.

\begin{table}[H]
  \centering
  \scriptsize
  \setlength{\tabcolsep}{4pt}
  \caption{Ablation on exploration schedule and discretization under the SD3.5-M 10-NFE protocol at 3000 iterations. The first row reports Dance-GRPO raw scores; later rows report absolute changes over the previous row. Higher is better for every metric.}
  \label{tab:schedule_discretization_ablation}
  \begin{tabular}{@{}lcccccc@{}}
    \toprule
    & \multicolumn{3}{c}{Training rewards} & \multicolumn{3}{c}{Held-out metrics} \\
    \cmidrule(lr){2-4}\cmidrule(l){5-7}
    Ablation & PickScore $\uparrow$ & CLIPScore $\uparrow$ & HPSv2.1 $\uparrow$ & ImageReward $\uparrow$ & UnifiedReward v2 $\uparrow$ & Aesthetics $\uparrow$ \\
    \midrule
    Dance-GRPO~\citep{xue2025dancegrpo} & 22.428 & 0.942 & 0.340 & 1.350 & 0.647 & 6.620 \\
    \quad + SDE-consistent discretization & \textcolor{red}{+0.826} & \textcolor{red}{+0.043} & \textcolor{red}{+0.037} & \textcolor{red}{+0.151} & \textcolor{red}{+0.003} & \textcolor{red}{+0.008} \\
    \qquad + logSNR exploration schedule & \textcolor{red}{+0.491} & \textcolor{red}{+0.053} & \textcolor{red}{+0.014} & \textcolor{red}{+0.114} & \textcolor{red}{+0.004} & \textcolor{red}{+0.024} \\
    \bottomrule
  \end{tabular}
\end{table}

Table~\ref{tab:schedule_discretization_ablation} shows that replacing Dance-GRPO's
Euler transition with the SDE-consistent finite-step transition improves all
reported metrics, and replacing the constant schedule with the logSNR-derived
schedule gives another consistent gain. Thus, both the finite-step transition
and exploration schedule contribute to reward optimization.

\paragraph{NFE and exploration strength.}
We additionally test whether \textsc{Precise} is robust to NFE and
exploration strength.

\begin{figure}[htb]
  \centering
  \begin{minipage}[t]{0.42\textwidth}
    \centering
    \includegraphics[height=1.38in]{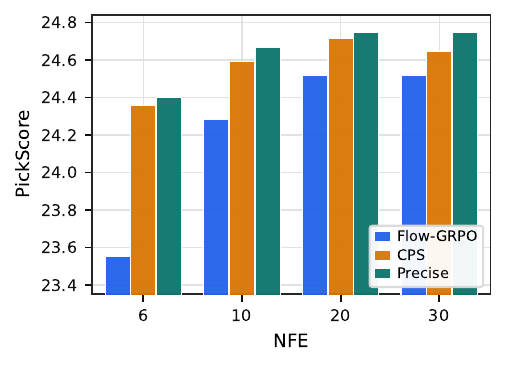}
  \end{minipage}\hfill
  \begin{minipage}[t]{0.56\textwidth}
    \centering
    \includegraphics[height=1.38in]{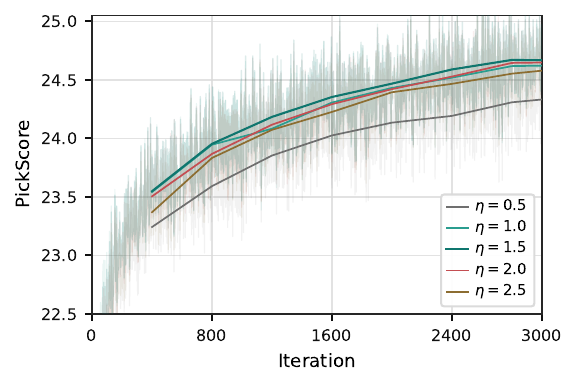}
  \end{minipage}
  \caption{Ablations on NFE and exploration strength. Left: NFE ablation on PickScore. Right: $\eta$ ablation on PickScore. Higher is better.}
  \label{fig:step_eta_ablations}
\end{figure}

The NFE ablation trains each sampler for $3000$ iterations at
$N\in\{6,10,20,30\}$ on PickScore. In the left panel of
Figure~\ref{fig:step_eta_ablations}, \textsc{Precise} leads at every tested NFE,
with the clearest margin in the aggressive 6-NFE regime. Flow-GRPO~\citep{liu2025flow}
degrades sharply at low NFE because of excess noise.

To test robustness to the exploration strength $\eta$, we sweep
$\eta\in\{0.5,1.0,1.5,2.0,2.5\}$ for \textsc{Precise} in the 10-NFE PickScore
setting. As shown in the right panel of
Figure~\ref{fig:step_eta_ablations}, performance is stable across
$\eta\in\{1.0,1.5,2.0\}$ and degrades gracefully at the more extreme values, due to lack of exploration or denoising instability. Thus, \textsc{Precise} is robust across a practical
range of NFE budgets and exploration strengths.

\section{Additional Qualitative Comparisons}
\label{app:qualitative_comparisons}

Figures~\ref{fig:qualitative_pickscore} and~\ref{fig:qualitative_pickscore_30}
qualitatively compare images sampled by Flow-GRPO, CPS, and \textsc{Precise}
under the two SD3.5-M protocols in Section~\ref{sec:sd35m_results}. In the
10-NFE setting,
\textsc{Precise} more reliably preserves global composition, local texture, and
prompt-specific attributes than Flow-GRPO and CPS,
while Flow-GRPO samples still exhibit noise artifacts after training. Under 30-NFE
training and evaluation, similar qualitative differences remain visible on identity,
attribute binding, and scene structure.

\begin{figure*}[htb]
  \centering
  \includegraphics[width=\textwidth]{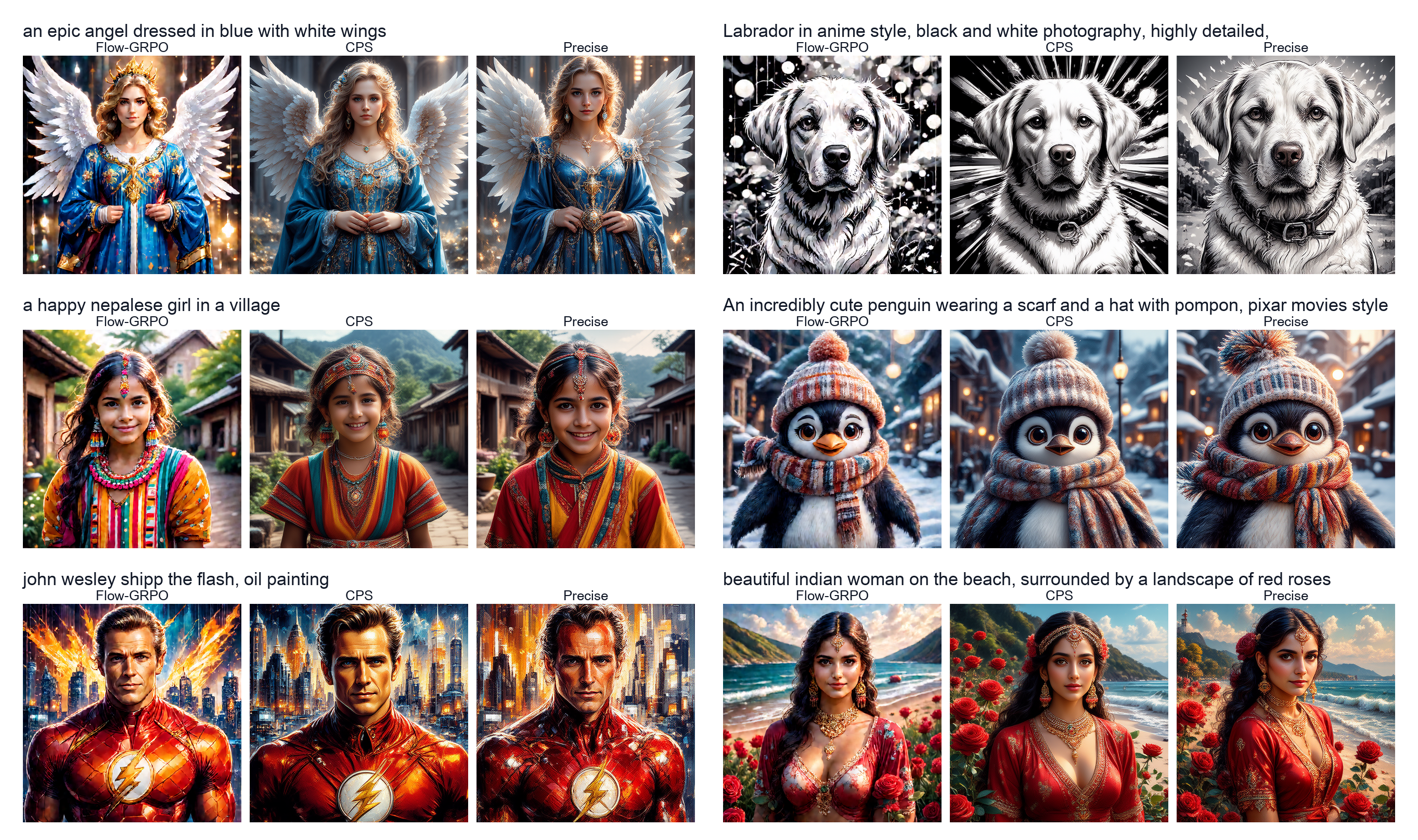}
  \caption{Qualitative examples under the 10-NFE protocol. \textsc{Precise} is the most consistent on global structure and attribute binding.}
  \label{fig:qualitative_pickscore}
\end{figure*}

\begin{figure*}[htb]
  \centering
  \includegraphics[width=\textwidth]{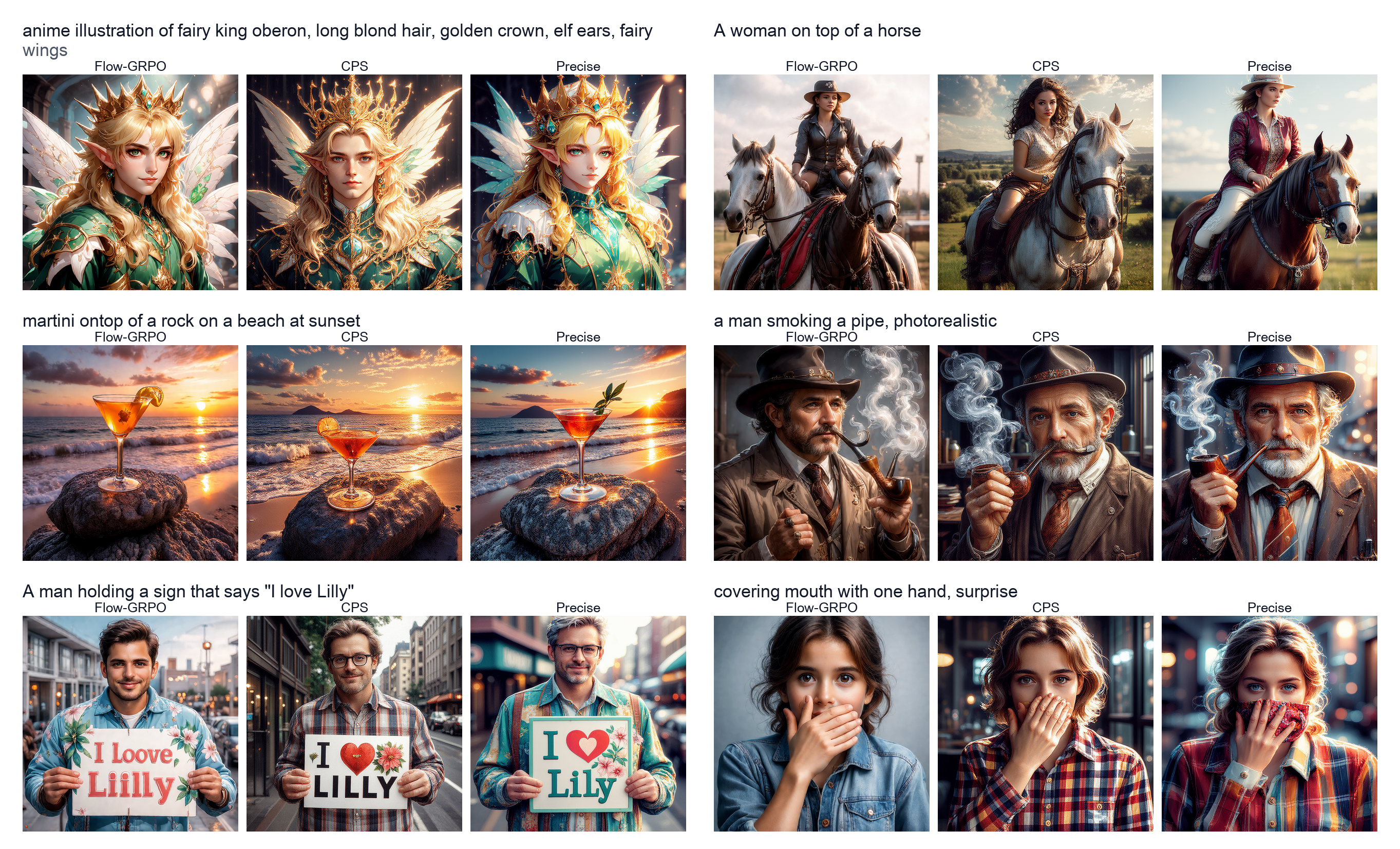}
  \caption{Qualitative examples under the 30-NFE protocol. \textsc{Precise} remains the most stable on identity and scene layout.}
  \label{fig:qualitative_pickscore_30}
\end{figure*}

\section{Conclusion}

We presented \textsc{Precise}, a stochastic sampler for online RL post-training
of flow-matching models. \textsc{Precise} treats the sampler as
part of the policy and addresses its two design choices: how to 
balance exploration and stability, and how to discretize the resulting SDE at
a small NFE budget. It combines a logSNR-based exploration
schedule with an SDE-consistent finite-step transition derived under a 
new frozen posterior-mean approximation.
Extensive experiments demonstrate that \textsc{Precise} consistently 
improves reward optimization over prior samplers, both in terms of final reward and training time.

\paragraph{Limitations.} The frozen posterior-mean approximation is local and less accurate where the clean-latent prediction changes quickly within a step. Each training configuration is run with a single seed, and our evaluation relies on automatic reward models rather than a human preference study.

\newpage

\bibliographystyle{omniweaving_arxiv}

\bibliography{refs}
\clearpage
\appendix

\section{Proofs and Extensions for the Sampler Analysis}

This appendix contains the formal derivations behind the sampler analysis in
the main text. The key point is local: posterior covariance controls how much
the clean-latent posterior mean moves, and once this mean is frozen, the reverse
SDE has an exact Gaussian transition. Lemma~\ref{lem:logsnr_decomp} gives the
first-order logSNR balance among the velocity term, score term, and stochastic
term. Lemma~\ref{lem:mean_jacobian} and
Proposition~\ref{prop:frozen_mean_error} bound the local error introduced by
freezing the posterior mean, Appendix~\ref{app:frozen_mean_diagnostics}
measures that error directly on SD3.5-M, and Theorem~\ref{cons:exact_update}
gives the exact finite-step transition of the frozen-posterior-mean SDE.

\subsection{Proof of Lemma~\ref{lem:logsnr_decomp}}
\label{app:proof_logsnr_decomp}

\begin{proof}
Over a reverse step from $t$ to $t-\Delta t$, write the one-step change as the sum of
three parts:
\begin{align*}
\text{(velocity term)} &\quad -\mathbf{u}_t \Delta t,
\\
\text{(score term)} &\quad \frac{1}{2}\varepsilon_t^2
\nabla_{\mathbf{z}}\log p_t(\mathbf{z}_t)\Delta t,
\\
\text{(stochastic term)} &\quad
\varepsilon_t\sqrt{\Delta t}\,\mathbf{w},
\qquad
\mathbf{w}\sim\mathcal{N}(\mathbf{0}, \mathbf{I}).
\end{align*}

The velocity term takes
\begin{equation*}
\mathbf{z}_t=(1-t)\mathbf{z}_0+t\boldsymbol{\epsilon}
\end{equation*}
to
\begin{equation*}
\mathbf{z}_{t-\Delta t}
=(1-(t-\Delta t))\mathbf{z}_0+(t-\Delta t)\boldsymbol{\epsilon},
\end{equation*}
so its logSNR increment is
\begin{align*}
\Delta\lambda_{\mathrm{vel}}
&=
\log\frac{(1-t+\Delta t)^2}{(t-\Delta t)^2}
- \log\frac{(1-t)^2}{t^2}
\\
&=
\frac{2}{t(1-t)}\Delta t + o(\Delta t).
\end{align*}

Next isolate the stochastic term while keeping the time index fixed at $t$.
This replaces the noise coefficient $t$ by
\begin{equation*}
\sqrt{t^2+\varepsilon_t^2\Delta t},
\end{equation*}
hence
\begin{align*}
\Delta\lambda_{\mathrm{sto}}
&=
\log\frac{t^2}{t^2+\varepsilon_t^2\Delta t}
\\
&=
-\frac{\varepsilon_t^2}{t^2}\Delta t + o(\Delta t).
\end{align*}

To isolate the score contribution, we freeze time at \(t\) and consider the
score drift together with its paired stochastic diffusion. With respect to the
frozen marginal \(p_t\), this pair forms a Langevin operator whose
Fokker--Planck adjoint annihilates \(p_t\). Hence the combined score-noise
substep has no first-order effect on the marginal. Consequently, the
first-order score contribution must cancel the isolated stochastic
perturbation:
\begin{equation*}
\Delta\lambda_{\mathrm{sco}}
=
-\Delta\lambda_{\mathrm{sto}}
=
\frac{\varepsilon_t^2}{t^2}\Delta t + o(\Delta t).
\end{equation*}
Together with the identity for the velocity term above, this proves the three
first-order contributions claimed in Lemma~\ref{lem:logsnr_decomp}.
\end{proof}

\subsection{Local stability of the frozen posterior-mean approximation}
\label{app:proof_frozen_mean_local}

The frozen posterior-mean approximation is controlled when the clean-latent
posterior mean changes slowly under local state and time perturbations. Let
\(\mathbf{m}_t(\mathbf{z})\triangleq
\mathbb{E}[\mathbf{z}_0\mid \mathbf{z}_t=\mathbf{z}]\) denote the clean-latent
posterior mean as a function of the current latent.

\begin{lemma}[Posterior-covariance control of the clean-latent posterior mean]
\label{lem:mean_jacobian}
For the flow-matching path
\(\mathbf{z}_t=(1-t)\mathbf{z}_0+t\boldsymbol{\epsilon}\), with
\(\boldsymbol{\epsilon}\sim\mathcal{N}(\mathbf{0},\mathbf{I})\), assume
regularity conditions that justify differentiating the posterior mean with
respect to the conditioning variable, such as a regular conditional density,
local domination, and finite second moments. Then
\begin{equation}
\nabla_{\mathbf{z}}\mathbf{m}_t(\mathbf{z})
=
\frac{1-t}{t^2}
\operatorname{Cov}(\mathbf{z}_0\mid \mathbf{z}_t=\mathbf{z}).
\label{eq:mean_jacobian}
\end{equation}
Consequently, on any convex local region where
\begin{equation*}
\lambda_{\max}\!\left(
\operatorname{Cov}(\mathbf{z}_0\mid \mathbf{z}_t=\mathbf{z})
\right)
\le
L_z\frac{t^2}{1-t},
\end{equation*}
this mean map is locally \(L_z\)-Lipschitz in state.
\end{lemma}

Lemma~\ref{lem:mean_jacobian} recovers the point-mass case as the
zero-covariance limit, where this mean is exactly constant. When the posterior
covariance is small but nonzero, it gives a local sensitivity bound that can be
converted into a one-step transition-law bound.

\begin{proposition}[One-step transition error under local regularity]
\label{prop:frozen_mean_error}
Consider one reverse step from \(t\) to \(t'=t-\Delta\) on a compact time
interval away from the endpoints. Let
\(\mathcal{K}_{t,t'}(\cdot\mid\mathbf{z}_t)\) denote the conditional
transition law of the original reverse SDE in Eq.~\eqref{eq:fm_sde_mean_form},
and let \(\widetilde{\mathcal{K}}_{t,t'}(\cdot\mid\mathbf{z}_t)\) denote the
conditional transition law of the frozen posterior-mean SDE obtained by
replacing \(\mathbf{m}_s(\mathbf{z}_s)\) with
\(\mathbf{m}_t(\mathbf{z}_t)\) throughout the step. Suppose the two processes
are synchronously coupled and remain in a local tube where the drift is
\(C_z\)-Lipschitz in the state and \(C_m\)-Lipschitz in the clean-latent
posterior mean. If
\begin{equation*}
\mathbb{E}\left\|
\mathbf{m}_s(\mathbf{Z}_s)-\mathbf{m}_t(\mathbf{z}_t)
\right\|
\le
\omega(s),
\qquad s\in[t',t],
\end{equation*}
then the transition laws satisfy
\begin{equation}
W_1\!\left(
\mathcal{K}_{t,t'}(\cdot\mid\mathbf{z}_t),
\widetilde{\mathcal{K}}_{t,t'}(\cdot\mid\mathbf{z}_t)
\right)
\le
C_m e^{C_z\Delta}
\int_{t'}^t \omega(s)\,\mathrm{d}s.
\label{eq:frozen_mean_error}
\end{equation}
In particular, if the posterior mean is \(L_s\)-Lipschitz in time and
\(L_z\)-Lipschitz in state along the step, and if
\(\mathbb{E}\|\mathbf{Z}_s-\mathbf{z}_t\|
\le B(t-s)+\bar{\varepsilon}\sqrt{t-s}\), then the right-hand side is at most
\begin{equation*}
C_m e^{C_z\Delta}
\left(
\frac{L_s+L_z B}{2}\Delta^2
+\frac{2L_z\bar{\varepsilon}}{3}\Delta^{3/2}
\right).
\end{equation*}
Thus the approximation is accurate on steps where the posterior mean is locally
stable in time and state.
\end{proposition}

\begin{proof}[Proof of Lemma~\ref{lem:mean_jacobian}]
Write \(a_t=1-t\). For a fixed \(t\in(0,1)\), the posterior density of
\(\mathbf{z}_0\) given \(\mathbf{z}_t=\mathbf{z}\) is proportional to
\begin{equation*}
p_0(\mathbf{x})
\exp\!\left(
-\frac{\|\mathbf{z}-a_t\mathbf{x}\|^2}{2t^2}
\right).
\end{equation*}
Let \(q_t(\mathbf{x}\mid\mathbf{z})\) denote this posterior density. Its
score with respect to the conditioning variable is
\begin{equation*}
\nabla_{\mathbf{z}}\log q_t(\mathbf{x}\mid\mathbf{z})
=
\frac{a_t\mathbf{x}-\mathbf{z}}{t^2}
-\nabla_{\mathbf{z}}\log p_t(\mathbf{z}).
\end{equation*}
Differentiating the posterior mean under the integral and using the standard
identity
\(\nabla_{\mathbf{z}}\mathbb{E}_{q_t}[f]
=\operatorname{Cov}_{q_t}(f,\nabla_{\mathbf{z}}\log q_t)\) gives
\begin{align*}
\nabla_{\mathbf{z}}\mathbf{m}_t(\mathbf{z})
&=
\operatorname{Cov}\!\left(
\mathbf{z}_0,
\frac{a_t\mathbf{z}_0-\mathbf{z}}{t^2}
-\nabla_{\mathbf{z}}\log p_t(\mathbf{z})
\;\middle|\; \mathbf{z}_t=\mathbf{z}
\right)
\\
&=
\frac{a_t}{t^2}
\operatorname{Cov}(\mathbf{z}_0\mid \mathbf{z}_t=\mathbf{z}),
\end{align*}
because the terms depending only on \(\mathbf{z}\) have zero posterior
covariance with \(\mathbf{z}_0\). Since \(a_t=1-t\), this proves
Eq.~\eqref{eq:mean_jacobian}. The Lipschitz statement follows by bounding the
operator norm of the Jacobian and applying the mean-value theorem on the
region.
\end{proof}

\begin{proof}[Proof of Proposition~\ref{prop:frozen_mean_error}]
Let \(\widetilde{\mathbf{Z}}_s\) denote the frozen posterior-mean process
coupled to \(\mathbf{Z}_s\) with the same Brownian path and the same initial
condition at time \(t\). The Brownian terms cancel in the synchronous coupling.
By the local Lipschitz assumptions, the coupled difference
\(\boldsymbol{\delta}_s\triangleq \mathbf{Z}_s-\widetilde{\mathbf{Z}}_s\)
satisfies the integral inequality
\begin{equation*}
\mathbb{E}\|\boldsymbol{\delta}_r\|
\le
\int_r^t C_z\mathbb{E}\|\boldsymbol{\delta}_s\|\,\mathrm{d}s
+
C_m\int_r^t
\mathbb{E}\left\|
\mathbf{m}_s(\mathbf{Z}_s)-\mathbf{m}_t(\mathbf{z}_t)
\right\|\,\mathrm{d}s,
\qquad r\in[t',t],
\end{equation*}
Gronwall's inequality and the definition of \(\omega\) give
\begin{equation*}
\mathbb{E}\|\boldsymbol{\delta}_{t'}\|
\le
C_m e^{C_z\Delta}
\int_{t'}^t \omega(s)\,\mathrm{d}s.
\end{equation*}
This endpoint coupling upper-bounds \(W_1\), proving
Eq.~\eqref{eq:frozen_mean_error}.

For the explicit Lipschitz corollary, decompose for \(s\in[t',t]\)
\begin{align*}
\left\|
\mathbf{m}_s(\mathbf{Z}_s)-\mathbf{m}_t(\mathbf{z}_t)
\right\|
&\le
\left\|
\mathbf{m}_s(\mathbf{Z}_s)-\mathbf{m}_t(\mathbf{Z}_s)
\right\|
+
\left\|
\mathbf{m}_t(\mathbf{Z}_s)-\mathbf{m}_t(\mathbf{z}_t)
\right\|
\\
&\le
L_s(t-s)+L_z\|\mathbf{Z}_s-\mathbf{z}_t\|.
\end{align*}
Combining this with the assumed path-increment bound and integrating over the
step gives the stated explicit rate.
\end{proof}

\subsection{Direct diagnostics of the frozen posterior-mean closure}
\label{app:frozen_mean_diagnostics}

\paragraph{Protocol.} We use SD3.5-M~\citep{esser2024scaling} at the 10-NFE
anchors of the training scheduler with $\eta=1.5$ and no classifier-free
guidance, on eight held-out PickScore prompts with one noise realization each.
The reference path $\mathbf{Z}^{\mathrm{ref}}$ integrates the reverse SDE in
Eq.~\eqref{eq:fm_sde} with 400 Euler--Maruyama substeps per coarse step,
re-evaluating the model at every substep. At each coarse anchor $t_i$, the
one-step branches start from the reference state
$\mathbf{Z}^{\mathrm{ref}}(t_i)$ and are driven by the reference's fine-grid
Gaussian innovations over the step: one-step Euler--Maruyama, on the same SDE
with $\eta=1.5$, freezes the velocity, score, and $\varepsilon_{t_i}$ and uses
the sum of the increments; one-step \textsc{Precise} freezes the anchor
clean-latent prediction and uses the weighted sum of the increments implied by
its exact transition. This is a shared discrete-innovation coupling rather
than a continuous Brownian-path coupling. Because Euler--Maruyama is evaluated
at $\eta=1.5$, its error here is not that of Flow-GRPO at its default $\eta$.
All distances are per-coordinate RMS.
Normalized quantities are formed within each trajectory and reported as mean
$\pm$ sample standard deviation over trajectories. We summarize the eight
interior steps $i=1,\dots,8$, excluding the entrance step at $t=1$, where
$\varepsilon_t$ diverges, and the terminal step. In the last row of
Table~\ref{tab:frozen_mean_diagnostics}, we first average the eight steps
within each trajectory and then aggregate across trajectories.

\paragraph{Within-step drift.} Let $\hat{\mathbf{z}}_0(\mathbf{z},s)$ denote
the model clean-latent prediction and $\Delta_i=t_i-t_{i+1}$. Along reference
trajectory $n$, define
\begin{equation*}
d_{n,i}(s)=\mathrm{RMS}\bigl[\hat{\mathbf{z}}_0(\mathbf{Z}^{\mathrm{ref}}_n(s),s)
-\hat{\mathbf{z}}_0(\mathbf{Z}^{\mathrm{ref}}_n(t_i),t_i)\bigr],
\qquad
A_{n,i}=\frac{1}{\Delta_i}\int_{t_{i+1}}^{t_i} d_{n,i}(s)\,\mathrm{d}s,
\end{equation*}
an empirical counterpart of the drift $\omega(s)$ in
Proposition~\ref{prop:frozen_mean_error}. We estimate $A_{n,i}$ by trapezoidal
integration over the substeps and divide it by the reference coarse-step
displacement
$R_{n,i}=\mathrm{RMS}[\mathbf{Z}^{\mathrm{ref}}_n(t_i)-\mathbf{Z}^{\mathrm{ref}}_n(t_{i+1})]$.

\paragraph{One-step error.} For branch $b$, let
$E_{b,i}=\mathrm{RMS}[\mathbf{Z}_b(t_{i+1})-\mathbf{Z}^{\mathrm{ref}}(t_{i+1})]$.
A \textsc{Precise}/Euler error ratio below one favors \textsc{Precise}.

\paragraph{Results.} Table~\ref{tab:frozen_mean_diagnostics} reports both
diagnostics. Averaged over a step, the clean-latent prediction drifts by
$16$--$33\%$ of the coarse-step displacement, and the drift at the end of a
step is larger, so
Assumption~\ref{ass:frozen_mean} is an approximation at 10 NFE rather than an
identity. Nevertheless, \textsc{Precise} is closer to the reference than
Euler--Maruyama on all 64 trajectory--step pairs, and the error ratio is below
one for every trajectory, with a trajectory-level mean of $0.11$. Both the
drift and the error ratio are largest at the last interior step, which is
large relative to its noise level.

\begin{table}[htb]
  \centering
  \scriptsize
  \setlength{\tabcolsep}{4pt}
  \caption{Direct diagnostics of the frozen posterior-mean closure on SD3.5-M
  (10-NFE anchors, $\eta=1.5$, eight trajectories). Entries are
  mean$\pm$standard deviation over trajectories.}
  \label{tab:frozen_mean_diagnostics}
  \begin{tabular}{@{}ccccc@{}}
    \toprule
    Step $i$ & Drift $A_i/R_i$ & \textsc{Precise} error $E_i$ & Euler error $E_i$ & \textsc{Precise}/Euler ratio \\
    \midrule
    1 & 0.2240{\scriptsize$\pm$0.0932} & 0.0159{\scriptsize$\pm$0.0071} & 1.0895{\scriptsize$\pm$0.0038} & 0.0146{\scriptsize$\pm$0.0065} \\
    2 & 0.1801{\scriptsize$\pm$0.0518} & 0.0146{\scriptsize$\pm$0.0043} & 0.5407{\scriptsize$\pm$0.0013} & 0.0270{\scriptsize$\pm$0.0080} \\
    3 & 0.1595{\scriptsize$\pm$0.0415} & 0.0149{\scriptsize$\pm$0.0040} & 0.3864{\scriptsize$\pm$0.0015} & 0.0386{\scriptsize$\pm$0.0103} \\
    4 & 0.1603{\scriptsize$\pm$0.0416} & 0.0182{\scriptsize$\pm$0.0047} & 0.3290{\scriptsize$\pm$0.0019} & 0.0554{\scriptsize$\pm$0.0144} \\
    5 & 0.1677{\scriptsize$\pm$0.0490} & 0.0245{\scriptsize$\pm$0.0072} & 0.3161{\scriptsize$\pm$0.0027} & 0.0775{\scriptsize$\pm$0.0233} \\
    6 & 0.1901{\scriptsize$\pm$0.0514} & 0.0381{\scriptsize$\pm$0.0100} & 0.3355{\scriptsize$\pm$0.0059} & 0.1140{\scriptsize$\pm$0.0316} \\
    7 & 0.2387{\scriptsize$\pm$0.0603} & 0.0727{\scriptsize$\pm$0.0177} & 0.3947{\scriptsize$\pm$0.0122} & 0.1856{\scriptsize$\pm$0.0505} \\
    8 & 0.3259{\scriptsize$\pm$0.0659} & 0.1852{\scriptsize$\pm$0.0272} & 0.4909{\scriptsize$\pm$0.0275} & 0.3809{\scriptsize$\pm$0.0765} \\
    \midrule
    Mean over $i$ & 0.2058{\scriptsize$\pm$0.0430} & 0.0480{\scriptsize$\pm$0.0082} & 0.4854{\scriptsize$\pm$0.0065} & 0.1117{\scriptsize$\pm$0.0240} \\
    \bottomrule
  \end{tabular}
\end{table}

\subsection{Proof of Theorem~\ref{cons:exact_update}}
\label{app:proof_exact_update}

\begin{proof}
We prove the exact transition of the frozen-posterior-mean SDE.
Fix a step from $t$ to $t' < t$ and write
\begin{equation*}
\mathbf{m}\triangleq \mathbf{m}_t(\mathbf{z}_t).
\end{equation*}
Under Assumption~\ref{ass:frozen_mean}, the state-dependent value
$\mathbf{m}_s(\mathbf{z}_s)$ is replaced by this fixed value throughout the
interval $[t', t]$.
Equation~\eqref{eq:fm_sde_mean_form} then reduces to the linear SDE
\begin{equation}
\mathrm{d}\mathbf{z}_s
=
\left(
\frac{\mathbf{z}_s-\mathbf{m}}{s}
+ \frac{\varepsilon_s^2}{2s^2}
\bigl(\mathbf{z}_s-(1-s)\mathbf{m}\bigr)
\right)\mathrm{d}s
+ \varepsilon_s\,\mathrm{d}\mathbf{w}_s.
\label{eq:appendix_frozen_mean_sde}
\end{equation}

Introduce the descending-time parameter
\begin{equation*}
\tau \in [0, t-t'],
\qquad
s(\tau)\triangleq t-\tau,
\end{equation*}
and define the normalized residual
\begin{equation*}
\mathbf{r}_\tau
\triangleq
\frac{\mathbf{z}_{s(\tau)}-(1-s(\tau))\mathbf{m}}{s(\tau)}.
\end{equation*}
Because $s(\tau)$ decreases with $\tau$, we have $\mathrm{d}s=-\mathrm{d}\tau$.
Using
\begin{equation*}
\mathbf{z}_{s(\tau)}=(1-s(\tau))\mathbf{m}+s(\tau)\mathbf{r}_\tau
\end{equation*}
and substituting into Eq.~\eqref{eq:appendix_frozen_mean_sde}, a direct
calculation gives
\begin{equation}
\mathrm{d}\mathbf{r}_\tau
=
-\frac{\varepsilon_{s(\tau)}^2}{2s(\tau)^2}\mathbf{r}_\tau\,\mathrm{d}\tau
+ \frac{\varepsilon_{s(\tau)}}{s(\tau)}\,\mathrm{d}\mathbf{B}_\tau,
\label{eq:appendix_residual_ou}
\end{equation}
where $\mathbf{B}_\tau$ denotes the Brownian driver in descending time.

Equation~\eqref{eq:appendix_residual_ou} is a linear SDE. Its solution is
\begin{align*}
\mathbf{r}_\tau
&=
\exp\!\left(
-\frac{1}{2}\int_0^\tau
\frac{\varepsilon_{s(u)}^2}{s(u)^2}\,\mathrm{d}u
\right)\mathbf{r}_0
\\
&\quad+
\int_0^\tau
\exp\!\left(
-\frac{1}{2}\int_u^\tau
\frac{\varepsilon_{s(v)}^2}{s(v)^2}\,\mathrm{d}v
\right)
\frac{\varepsilon_{s(u)}}{s(u)}\,\mathrm{d}\mathbf{B}_u.
\end{align*}
The stochastic integral is Gaussian with mean zero. By It\^o isometry, its
covariance is
\begin{align*}
\int_0^\tau
\exp\!\left(
-\int_u^\tau \frac{\varepsilon_{s(v)}^2}{s(v)^2}\,\mathrm{d}v
\right)
\frac{\varepsilon_{s(u)}^2}{s(u)^2}\,\mathrm{d}u
\mathbf{I}
&=
\left(
1-\exp\!\left(
-\int_0^\tau \frac{\varepsilon_{s(u)}^2}{s(u)^2}\,\mathrm{d}u
\right)
\right)\mathbf{I}.
\end{align*}
Hence
\begin{equation*}
\mathbf{r}_\tau
=
e^{-A(s(\tau),t)/2}\mathbf{r}_0
+ \sqrt{1-e^{-A(s(\tau),t)}}\,\mathbf{w},
\end{equation*}
where $\mathbf{w}\sim\mathcal{N}(\mathbf{0}, \mathbf{I})$ is independent of
$\mathbf{r}_0$, and
\begin{equation*}
A(s(\tau),t)=\int_{s(\tau)}^t \frac{\varepsilon_u^2}{u^2}\,\mathrm{d}u.
\end{equation*}

Now
\begin{equation*}
\mathbf{r}_0
=
\frac{\mathbf{z}_t-(1-t)\mathbf{m}}{t}.
\end{equation*}
Evaluating the solution at $\tau=t-t'$ so that $s(\tau)=t'$ and then
transforming back to $\mathbf{z}_{t'}=(1-t')\mathbf{m}+t'\mathbf{r}_{t-t'}$
gives
\begin{align*}
\mathbf{z}_{t'}
&=
(1-t')\mathbf{m}
+ \frac{t'}{t}e^{-A(t',t)/2}
\bigl(\mathbf{z}_t-(1-t)\mathbf{m}\bigr)
+ t'\sqrt{1-e^{-A(t',t)}}\,\mathbf{w},
\end{align*}
which is exactly the claimed formula.
\end{proof}

\section{Why CPS does not guarantee consistency}
\label{app:cps_variance}

This appendix expands the CPS covariance argument in
Section~\ref{sec:analysis}. With perfect prediction,
Flow-GRPO~\citep{liu2025flow} is Euler-Maruyama applied to the oracle reverse
SDE, and \textsc{Precise} has the same first-order expansion, so both share the
correct continuous-time limit as the step size goes to zero. CPS~\citep{wang2025coefficients}
addresses the point-mass failure mode of Euler-style transitions by preserving the
nominal signal and noise coefficients, but the same rule can collapse residual
posterior uncertainty for non-degenerate data distributions.

\paragraph{Posterior means discard residual uncertainty.}
Consider the flow-matching forward process from Eq.~\eqref{eq:fm_path},
\begin{equation*}
\mathbf{z}_t=(1-t)\mathbf{z}_0+t\boldsymbol{\epsilon},
\end{equation*}
and let
\begin{equation*}
\hat{\mathbf{z}}_0(t)\triangleq \mathbb{E}[\mathbf{z}_0\mid\mathbf{z}_t].
\end{equation*}
The posterior mean of the noise is
\begin{equation*}
\hat{\boldsymbol{\epsilon}}(t)
=
\mathbb{E}[\boldsymbol{\epsilon}\mid\mathbf{z}_t]
=
\frac{\mathbf{z}_t-(1-t)\hat{\mathbf{z}}_0(t)}{t}.
\end{equation*}

CPS~\citep{wang2025coefficients} applies a coefficient-preserving transition
\begin{equation*}
\mathbf{z}_{t'}^{\mathrm{CPS}}
=(1-t')\hat{\mathbf{z}}_0(t)
+k_1\hat{\boldsymbol{\epsilon}}(t)
+k_2\mathbf{w},
\qquad
k_1^2+k_2^2=t'^2,
\end{equation*}
where $\mathbf{w}\sim\mathcal{N}(\mathbf{0},\mathbf{I})$ is independent of
$(\mathbf{z}_0,\boldsymbol{\epsilon})$. Define
\begin{equation*}
\mathbf{L}\triangleq(1-t')\mathbf{z}_0+k_1\boldsymbol{\epsilon}.
\end{equation*}
Then the deterministic component of the CPS transition is the posterior mean of the
coefficient-preserved random variable $\mathbf{L}$:
\begin{equation*}
(1-t')\hat{\mathbf{z}}_0(t)+k_1\hat{\boldsymbol{\epsilon}}(t)
=
\mathbb{E}[\mathbf{L}\mid\mathbf{z}_t].
\end{equation*}

By the law of total covariance,
\begin{equation*}
\operatorname{Cov}(\mathbf{z}_{t'}^{\mathrm{CPS}})
=
\operatorname{Cov}(\mathbb{E}[\mathbf{L}\mid\mathbf{z}_t])
+k_2^2\mathbf{I}
=
\operatorname{Cov}(\mathbf{L})
-\mathbb{E}_{\mathbf{z}_t}
\left[
\operatorname{Cov}(\mathbf{L}\mid\mathbf{z}_t)
\right]
+k_2^2\mathbf{I}.
\end{equation*}
Since $\mathbf{z}_0$ and $\boldsymbol{\epsilon}$ are independent,
\begin{equation*}
\operatorname{Cov}(\mathbf{L})+k_2^2\mathbf{I}
=
(1-t')^2\operatorname{Cov}(\mathbf{z}_0)
+(k_1^2+k_2^2)\mathbf{I}
=
(1-t')^2\operatorname{Cov}(\mathbf{z}_0)+t'^2\mathbf{I}.
\end{equation*}
Thus CPS has covariance
\begin{equation*}
\operatorname{Cov}(\mathbf{z}_{t'}^{\mathrm{CPS}})
\preceq
(1-t')^2\operatorname{Cov}(\mathbf{z}_0)+t'^2\mathbf{I},
\end{equation*}
with strict trace contraction whenever
$\mathbb{E}_{\mathbf{z}_t}
\left[
\operatorname{tr}\operatorname{Cov}(\mathbf{L}\mid\mathbf{z}_t)
\right]>0$.
Thus coefficient preservation fixes the point-mass case, but it does not by
itself preserve the full marginal law. Whenever $\mathbf{z}_t$ does not
determine the coefficient-preserved random variable $\mathbf{L}$, replacing
$\mathbf{L}$ by its posterior mean removes the conditional covariance that
should remain in the next-step marginal.

\paragraph{Double-ring example.}
The double-ring example in Figure~\ref{fig:sampler_design} visualizes this
contraction as an inward bias. We use an equal-mass 2D distribution supported
on rings of radii $0.5$ and $1.0$, and approximate the oracle clean-latent
posterior mean
$\hat{\mathbf{z}}_0(t)=\mathbb{E}[\mathbf{z}_0\mid\mathbf{z}_t]$ with a dense
discrete support on both rings. We use $\eta=0.7$ for
Flow-GRPO~\citep{liu2025flow} and CPS~\citep{wang2025coefficients}, and
$\eta=1.5$ for \textsc{Precise}, matching the main experimental protocol.
As shown in Figure~\ref{fig:sampler_design}, CPS samples remain biased toward
the inner ring even at $N=80$.

Figure~\ref{fig:double_ring_outer_mass} isolates the large-$N$ regime by
tracking the CPS outer-ring mass as the NFE increases to
$N=1280$. The target outer-ring mass is $0.5$, but the curve does not approach that value as $N$ increases, demonstrating that the coefficient-preserving rule does not guarantee convergence to the correct distribution.

\begin{figure}[t]
  \centering
  \includegraphics[width=0.56\linewidth]{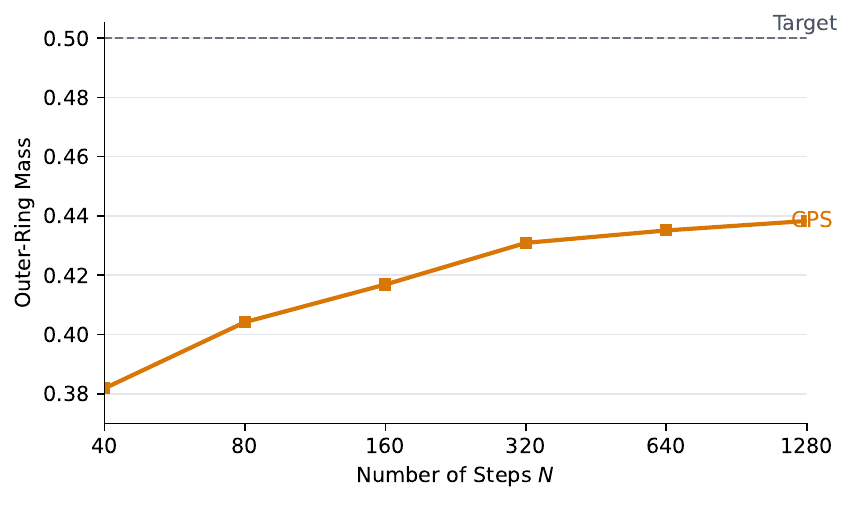}
  \caption{CPS outer-ring mass on the equal-mass double-ring distribution. The
  target value is $0.5$, but CPS remains biased toward the inner ring as NFE
  increases.}
  \label{fig:double_ring_outer_mass}
\end{figure}

\section{External Assets, Licenses, and Terms of Use}
\label{app:asset_licenses}

We use existing code, prompt datasets, pretrained backbones, reward models, and
evaluation metrics. Table~\ref{tab:asset_licenses} lists the original sources
and licenses. We use these assets for research evaluation and do not
redistribute third-party model weights or datasets beyond their stated terms.

Following the Flow-GRPO codebase~\citep{liu2025flow}, reported CLIPScore values
are the raw image-text similarity multiplied by $10/3$. This scale does not affect training dynamics since rewards are normalized to compute advantages.

\begin{table}[ht]
  \centering
  \scriptsize
  \setlength{\tabcolsep}{3pt}
  \caption{External assets used in this work.}
  \label{tab:asset_licenses}
  \begin{tabular}{@{}p{0.25\linewidth}p{0.42\linewidth}p{0.23\linewidth}@{}}
    \toprule
    Asset & Source & License \\
    \midrule
    Flow-GRPO codebase & Official Flow-GRPO repository~\citep{liu2025flow} & MIT \\
    Stable Diffusion 3.5 Medium & Stability AI SD3.5-M checkpoint~\citep{esser2024scaling} & Stability AI Community License \\
    FLUX.2 Klein 4B Base & Black Forest Labs FLUX.2 Klein 4B Base~\citep{blackforestlabs2025flux2,blackforestlabs2026flux2klein} & Apache-2.0 \\
    GenEval & GenEval benchmark~\citep{ghosh2023geneval} & MIT \\
    PickScore & PickScore benchmark and implementation~\citep{kirstain2023pick} & MIT \\
    CLIPScore / CLIP & CLIPScore with OpenAI CLIP (\texttt{clip-vit-large-patch14})~\citep{hessel2021clipscore,radford2021learning} & MIT \\
    HPSv2.1 & HPSv2.1 checkpoint and code~\citep{wu2023human} & Apache-2.0 \\
    ImageReward & ImageReward package and checkpoint~\citep{xu2023imagereward} & Apache-2.0 \\
    UnifiedReward v2 & UnifiedReward checkpoint~\citep{wang2025unified} & MIT \\
    Aesthetic predictor & LAION-Aesthetics Predictor V1~\citep{laion2022aestheticpredictor} & MIT \\
    \bottomrule
  \end{tabular}
\end{table}

\section{Broader Societal Impacts}
\label{app:societal_impacts}

This work can make visual generation systems more controllable and efficient to
adapt, supporting creative, accessibility, and educational uses. However,
stronger post-training can also increase misuse risks, amplify reward-model
bias, reduce diversity, or weaken safeguards, so deployment should include
audits, safety filters, provenance mechanisms where appropriate, and evaluations
for bias, diversity, memorization, and misuse-sensitive content.

\end{document}